\documentclass[journal]{IEEEtran}
\usepackage{setspace}
\usepackage{amssymb}
\usepackage{amsmath}
\usepackage{graphicx}
\usepackage{booktabs}
\usepackage{cite}
\usepackage{url}
\usepackage{enumerate}
\usepackage{tikz}
\usepackage[ruled,vlined]{algorithm2e}
\usetikzlibrary{shapes,decorations}
\usepgflibrary{shapes.geometric}
\usepackage{multirow}
\usepackage{comment}
\usepackage[usenames,dvipsnames]{pstricks}
\usepackage{epsfig}
\usepackage{subfigure}
\usepackage{enumerate}
\usepackage{bm}

\newtheorem{remark}{Remark}
\newtheorem{theorem}{Theorem}
\newtheorem{lemma}{Lemma}

\allowdisplaybreaks

\long\def\symbolfootnote[#1]#2{\begingroup%
 \def\thefootnote{\fnsymbol{footnote}}\footnote[#1]{#2}\endgroup}

\usepackage{cite}

\allowdisplaybreaks

\begin{document}

\def\Alg{\mathsf{Alg}} 
\def\reals{\mathbb{R}} 
\def\R{\mathbb{R}}
\def\Pr{\mathbb{P}}
\def\integers{\mathbb{Z}} 
\def\Z{\mathbb{Z}}
\def\rationals{\mathbb{Q}} 
\def\Q{\mathbb{Q}}
\def\naturals{\mathbb{N}} 
\def\N{\mathbb{N}}
\def\Xb {{\mathbf{X}}}
\def\Xf {\mathfrak{X}}
\def\xb {{\mathbf{x}}}
\def\yb {{\mathbf{y}}}
\def\Gc {{\mc{G}}}
\def\Vc {{\mc{V}}}
\def\Wc {{\mc{W}}}
\def\Uc {{\cal U}}
\def\Ec {{\cal E}}
\def\Xc {{\cal X}}
\def\Tc {{\cal T}}
\def\Ic {{\cal I}}
\def\Jc {{\cal J}}
\def\Ac {{\cal A}}
\def\C {{\cal C}}

\title{High-Dimensional Screening Using Multiple Grouping of Variables}
\author{Divyanshu Vats
\thanks{Copyright (c) 2013 IEEE. Personal use of this material is permitted. However, permission to use this material for any other purposes must be obtained from the IEEE by sending a request to pubs-permissions@ieee.org.}
\thanks{Divyanshu Vats is with the Department of Electrical and Computer Engineering, Rice University, Houston, TX, 77098 (email: dvats@rice.edu).}
\thanks{This work was supported by an Institute for Mathematics and its Applications (IMA) postdoctoral fellowship.}
}
\date{}

\maketitle

\begin{abstract}
Screening is the problem of finding a superset of the set of non-zero entries in an unknown $p$-dimensional vector $\beta^*$ given $n$ noisy observations.  Naturally, we want this superset to be as small as possible.  We propose a novel framework for screening, which we refer to as Multiple Grouping (MuG), that groups variables, performs variable selection over the groups, and repeats this process multiple number of times to estimate a sequence of sets that contains the non-zero entries in $\beta^*$.  Screening is done by taking an intersection of all these estimated sets.  The MuG framework can be used in conjunction with any group based variable selection algorithm.  
In the high-dimensional setting, where $p \gg n$, we show that when MuG is used with the group Lasso estimator, screening can be consistently performed without using any tuning parameter.  Our numerical simulations clearly show the merits of using the MuG framework in practice.
\end{abstract}

\begin{IEEEkeywords}
Screening, Lasso, Group Lasso, Variable Selection, Multiple Grouping, Randomized Algorithms
\end{IEEEkeywords}


\section{Introduction}

Let $\beta^* \in \R^{p \times 1}$ be an \textit{unknown} $p$-dimensional sparse vector with $k < p$ non-zero entries.  Let $S^*$ denote the support of $\beta^*$, i.e., the location of the non-zero entries in $\beta^*$.  Let $y \in \R^{n \times 1}$ be a \emph{known} $n$-dimensional vector that captures information about $\beta^*$ using the linear model
\begin{equation}
 y = X \beta^* + w \,, \label{eq:linreg}
\end{equation}
where $X \in \R^{n\times p}$ is a \emph{known} design matrix and $w$ is measurement noise.  Throughout this paper, we assume that $p > n$, i.e., the number of measurements available is smaller than the ambient dimensionality of $\beta^*$.  

Equation~(\ref{eq:linreg}) is well studied in the literature owing to its application in many real world problems.  For example, in compressive sensing, it is of interest to measure a signal $\beta^*$ using only a few measurements with a suitable choice of the design matrix $X$ \cite{candes2006robust,donoho2006compressed}.  Given gene expression data, where typically the number of observations $n$ is much smaller than the total number of genes $p$, it is of interest to study the relationships between genes \cite{wille2004sparse}.  These relationships are captured by the non-zero entries of the vector $\beta^*$.  A similar problem of estimating relationships arises when modeling economic data \cite{fan2011sparse}.

In this paper, we study the problem of finding a superset $\overline{S}$ of the support of $S^*$ so that the number of elements in $\overline{S}$ is less than $n$ and as close to $k$ as possible.  In the literature, this problem is often referred to as \textit{variable screening} or simply \textit{screening}.  Algorithms for screening are useful in reducing the dimensionality of $\beta^*$ from $p$ to $|\overline{S}|$.  This allows practitioners to focus subsequent analysis on a much smaller set rather than dealing with the large set of $p$ variables.

\subsection{Main Contributions}

\noindent
\textbf{Multiple Grouping (MuG):} We propose a general framework for screening that groups variables, performs variable selection over the groups, and repeats this process multiple number of times over different choices of the groupings.  The final estimated superset is the \emph{intersection} of the supports estimated over each grouping.  We refer to our framework as Multiple Grouping (MuG).  The main intuition behind MuG is that if a variable $v$ is selected in one iteration, it may not be selected in another iteration since the variable may be grouped with other variables that are all zero.  Figure~\ref{fig:mugframe} illustrates MuG using a simple example. 

\smallskip

\noindent
\textbf{Tuning-Free Screening:}  The MuG framework can be used in conjunction with \textit{any} group based variable selection algorithm.  We study the application of MuG with group Lasso \cite{GroupLassoYuanLin2006}, which uses a modification of the popular Lasso algorithm \cite{TibshiraniLasso1994} to perform variable selection over groups.  Using properties of the Lasso and the group Lasso, we show that when $p > n$, MuG with group Lasso can perform screening \textit{without using a tuning parameter} in such a way that $|\overline{S}| < n$.  This property of MuG is extremely useful, for example, in analyzing gene expression data where it is typical for $p$ to be of the order of thousands and $n$ to be of the order of hundreds.  Moreover, we identify conditions under which MuG is high-dimensional consistent so that $P\left(S^* \subseteq \overline{S}\right) \rightarrow 1$ as $n,p \rightarrow \infty$.

\subsection{Related Work}
\label{subsec:screening}

Several algorithms have been proposed in the literature for estimating the support of $\beta^*$, see \cite{TibshiraniLasso1994,tropp2007signal,ZouAdaptiveLasso2006,candes2008enhancing,wasserman2009high,GeerBuhlZhouAdaptive2011,zhang2011adaptive,belloni2011square} for some examples and \cite{buhlmann2011statistics} for a comprehensive review.  The performance of all the known algorithms depend on the choice of a tuning parameter that controls the cardinality of the estimated support.  To estimate the true support in the high-dimensional regime, where $p \gg n$, tuning parameters, also referred to as regularization parameters, may be chosen using stability selection \cite{meinshausen2010stability} or extended Bayesian information criterion \cite{chen2008extended}.  However, both these methods depend on other parameters which are difficult to select in practice.  The MuG framework can be used to eliminate a majority of the $p$ variables without using any tuning parameter and then standard model selection algorithms may be used over the remaining set variables.

It has been observed that when using cross-validation to select the tuning parameter, under appropriate conditions, the estimated support can be a superset of $S^*$.  However, the cardinality of the estimated support can be quite large in practice, making cross-validation based methods inappropriate for screening.

Reference \cite{fan2008sure} outlines a screening algorithm, referred to as sure independence screening (SIS), that\footnote{Assuming the columns of $X$ are normalized so that $||X_i||_2 / \sqrt{n} = 1$.} thresholds $|X^T y|$ to find the superset $\overline{S}$ of $S^*$.  Extensions of SIS have been proposed in \cite{fan2009ultrahigh,fan2010sure,fan2011nonparametric,ke2012covariance}.  The performance of SIS is sensitive to the choice of the threshold and an appropriate choice of the threshold depends on the unknown parameters of the model in~(\ref{eq:linreg}).  The main advantage of MuG over SIS is that, when $p > n$, screening may be done \textit{without using a tuning parameter or a threshold}.  Moreover, our numerical simulations clearly show that the MuG framework can discard more variables when compared to SIS.  However, we note that SIS is computationally fast, with time complexity $O(np)$, and can be used in conjunction with MuG to trade-off computational complexity and accuracy.

Recent works in \cite{tibshirani2011strong,el2011safe,xiang2011learning,xiangfast2012} have analyzed the solutions of the Lasso to derive rules for discarding variables when solving the Lasso for a \textit{particular} tuning parameter.  Our work differs from this work since we perform screening to find a superset of the true support $S^*$.  Regardless, when using the Lasso and the group Lasso with the MuG framework, the algorithms in \cite{tibshirani2011strong,el2011safe,xiang2011learning,xiangfast2012} can be used to improve the computational complexity of solving the Lasso and group Lasso problems.

Another approach to tuning-free screening is to use properties of variable selection algorithms such as the Lasso.  It is known that the Lasso can select at most $\min\{n,p\}$ variables \cite{osborne2000Lasso,liu2009estimation}.  When $p > n$, this means that screening can be done using the Lasso by choosing a tuning parameter so that the cardinality of the estimated support is $n$.  In our proposed algorithm of using MuG with group Lasso (see Algorithm~\ref{alg:muggroupLasso}), we use the Lasso to do tuning-free screening and then use the group Lasso multiple number of times to further screen for variables.  Using the property that the group Lasso can only select at most $n$ groups, we are able to perform tuning-free screening.  Our numerical simulations clearly show the improvements of using MuG versus simply using Lasso for screening.

\subsection{Paper Organization}

The rest of the paper is organized as follows.

\begin{itemize}

\item Section~\ref{sec:MuGOverview} presents the MuG framework with respect to an abstract group based variable selection algorithm.

\item Section~\ref{sec:MuGGroupLasso} shows how MuG can be used with the group Lasso estimator.

\item Section~\ref{sec:thprop} outlines conditions under which MuG leads to a screening algorithm that is high-dimensional consistent.

\item Section~\ref{sec:numsim} presents numerical simulations that show the advantages of using MuG in practice and compares MuG to other screening algorithms.

\item Section~\ref{sec:extensions} discusses some extensions of MuG.

\item Section~\ref{sec:summary} summarizes the paper.

\end{itemize}

\section{Overview of Multiple Grouping}
\label{sec:MuGOverview}

\begin{figure}
\centering
\scalebox{0.82} 
{
\begin{pspicture}(0,-1.9703125)(10.661875,1.9303125)
\definecolor{color1120b}{rgb}{0.7843137254901961,0.7843137254901961,0.7843137254901961}
\definecolor{color2b}{rgb}{0.5882352941176471,0.5882352941176471,0.5882352941176471}
\psframe[linewidth=0.0020,dimen=outer,fillstyle=solid,fillcolor=color2b](9.235782,-0.459375)(8.923905,-1.239375)
\psframe[linewidth=0.0020,dimen=outer,fillstyle=solid,fillcolor=color2b](9.235782,0.320625)(8.923905,-0.059375)
\psframe[linewidth=0.0020,dimen=outer,fillstyle=solid,fillcolor=color2b](9.235782,1.900625)(8.923905,1.540625)
\psframe[linewidth=0.0020,dimen=outer,fillstyle=solid,fillcolor=color1120b](2.4839058,-0.2696875)(2.1157813,-0.5896875)
\psframe[linewidth=0.0020,dimen=outer,fillstyle=solid,fillcolor=color1120b](2.4839058,1.3303125)(2.1157813,1.0103126)
\psframe[linewidth=0.0020,dimen=outer,fillstyle=solid,fillcolor=color1120b](3.663906,1.3303125)(3.3039062,1.0103126)
\psframe[linewidth=0.0020,dimen=outer,fillstyle=solid,fillcolor=color1120b](0.8839062,1.3303125)(0.5439062,1.0103126)
\psline[linewidth=0.04cm](0.8930912,1.0061684)(0.8947498,1.3461643)
\psline[linewidth=0.04cm](1.2930864,1.004217)(1.2947448,1.3442131)
\psline[linewidth=0.04cm](1.6930816,1.0022659)(1.6947403,1.3422617)
\psline[linewidth=0.04cm](2.093077,1.0003145)(2.0947354,1.3403105)
\psline[linewidth=0.04cm](2.4930718,0.9983632)(2.4947307,1.3383594)
\psline[linewidth=0.04cm](2.8930676,0.9965095)(2.894726,1.3365055)
\psline[linewidth=0.04cm](3.2930627,0.9944606)(3.2947211,1.3344566)
\usefont{T1}{ptm}{m}{n}
\rput{0.066784024}(0.0013669556,0.0){\rput(0.67096484,1.1606132){1}}
\usefont{T1}{ptm}{m}{n}
\rput{0.066784024}(0.0013686086,-0.0012158098){\rput(1.0359471,1.1618387){2}}
\usefont{T1}{ptm}{m}{n}
\rput{0.066784024}(0.0013636274,-0.0016496289){\rput(1.4012544,1.1595358){3}}
\usefont{T1}{ptm}{m}{n}
\rput{0.066784024}(0.0013670726,-0.0021593997){\rput(1.8476645,1.1600491){4}}
\usefont{T1}{ptm}{m}{n}
\rput{0.066784024}(0.001362344,-0.0026425668){\rput(2.2506216,1.1604383){5}}
\usefont{T1}{ptm}{m}{n}
\rput{0.066784024}(0.0013660925,-0.0030707463){\rput(2.619532,1.16094){6}}
\usefont{T1}{ptm}{m}{n}
\rput{0.066784024}(0.0013636098,-0.0035344297){\rput(3.0195246,1.1613908){7}}
\usefont{T1}{ptm}{m}{n}
\rput{0.066784024}(0.0013676978,-0.004065511){\rput(3.4667175,1.1618199){8}}
\psframe[linewidth=0.04,dimen=outer](3.7039063,1.3703125)(0.5039062,0.9703125)
\psline[linewidth=0.03cm,arrowsize=0.05291667cm 2.0,arrowlength=1.4,arrowinset=0.4]{<->}(0.5239062,0.8503125)(1.2839062,0.8503125)
\psline[linewidth=0.03cm,arrowsize=0.05291667cm 2.0,arrowlength=1.4,arrowinset=0.4]{<->}(1.3239062,1.4503125)(2.0839062,1.4503125)
\psline[linewidth=0.03cm,arrowsize=0.05291667cm 2.0,arrowlength=1.4,arrowinset=0.4]{<->}(2.1239064,0.8503125)(2.8839061,0.8503125)
\psline[linewidth=0.03cm,arrowsize=0.05291667cm 2.0,arrowlength=1.4,arrowinset=0.4]{<->}(2.903906,1.4503125)(3.663906,1.4503125)
\usefont{T1}{ptm}{m}{n}
\rput(0.8023437,0.5821875){$G_1^{1}$}
\usefont{T1}{ptm}{m}{n}
\rput(1.6623437,1.7121875){$G_2^{1}$}
\usefont{T1}{ptm}{m}{n}
\rput(2.4823437,0.5803125){$G_3^{1}$}
\usefont{T1}{ptm}{m}{n}
\rput(3.2623436,1.7121875){$G_4^{1}$}
\psframe[linewidth=0.04,dimen=outer](9.2639065,1.9303125)(8.883906,-1.2696875)
\psline[linewidth=0.04cm](8.903907,1.5303125)(9.243906,1.5303125)
\psline[linewidth=0.04cm](8.903907,1.1303124)(9.243906,1.1303124)
\psline[linewidth=0.04cm](8.903907,0.7303125)(9.243906,0.7303125)
\psline[linewidth=0.04cm](8.903907,0.3303125)(9.243906,0.3303125)
\psline[linewidth=0.04cm](8.903907,-0.0696875)(9.243906,-0.0696875)
\psline[linewidth=0.04cm](8.903907,-0.4496875)(9.243906,-0.4496875)
\psline[linewidth=0.04cm](8.903907,-0.8496875)(9.243906,-0.8496875)
\usefont{T1}{ptm}{m}{n}
\rput(9.058281,1.7003125){1}
\usefont{T1}{ptm}{m}{n}
\rput(9.076874,1.3203125){2}
\usefont{T1}{ptm}{m}{n}
\rput(9.062187,0.9203125){3}
\usefont{T1}{ptm}{m}{n}
\rput(9.048594,0.5203125){4}
\usefont{T1}{ptm}{m}{n}
\rput(9.051563,0.1203125){5}
\usefont{T1}{ptm}{m}{n}
\rput(9.040468,-0.2796875){6}
\usefont{T1}{ptm}{m}{n}
\rput(9.060469,-0.6796875){7}
\usefont{T1}{ptm}{m}{n}
\rput(9.047657,-1.0596875){8}
\psframe[linewidth=0.0020,dimen=outer,fillstyle=solid,fillcolor=color1120b](3.663906,-0.2696875)(3.3039062,-0.5896875)
\psframe[linewidth=0.0020,dimen=outer,fillstyle=solid,fillcolor=color1120b](0.8839062,-0.2696875)(0.5439062,-0.5896875)
\psline[linewidth=0.04cm](0.8930912,-0.5938316)(0.8947498,-0.2538356)
\psline[linewidth=0.04cm](1.2930864,-0.5957829)(1.2947448,-0.2557869)
\psline[linewidth=0.04cm](1.6930816,-0.5977341)(1.6947403,-0.2577383)
\psline[linewidth=0.04cm](2.093077,-0.5996855)(2.0947354,-0.2596895)
\psline[linewidth=0.04cm](2.4930718,-0.6016368)(2.4947307,-0.2616408)
\psline[linewidth=0.04cm](2.8930676,-0.6034905)(2.894726,-0.2634945)
\psline[linewidth=0.04cm](3.2930627,-0.6055394)(3.2947211,-0.2655434)
\usefont{T1}{ptm}{m}{n}
\rput{0.066784024}(0.0,0.0){\rput(0.6719593,-0.43938312){1}}
\usefont{T1}{ptm}{m}{n}
\rput{0.066784024}(0.0,-0.0012704926){\rput(1.0750521,-0.4382032){3}}
\usefont{T1}{ptm}{m}{n}
\rput{0.066784024}(0.0,-0.0017436575){\rput(1.4878678,-0.4404402){2}}
\usefont{T1}{ptm}{m}{n}
\rput{0.066784024}(0.0,-0.0021492483){\rput(1.8280215,-0.43998587){6}}
\usefont{T1}{ptm}{m}{n}
\rput{0.066784024}(0.0,-0.0026469948){\rput(2.2534878,-0.43956158){5}}
\usefont{T1}{ptm}{m}{n}
\rput{0.066784024}(0.0,-0.0030959293){\rput(2.6423907,-0.43905228){7}}
\usefont{T1}{ptm}{m}{n}
\rput{0.066784024}(0.0,-0.0035559712){\rput(3.0448885,-0.43858206){4}}
\usefont{T1}{ptm}{m}{n}
\rput{0.066784024}(0.0,-0.0040466143){\rput(3.4495733,-0.43818015){8}}
\psframe[linewidth=0.04,dimen=outer](3.7039063,-0.2296875)(0.5039062,-0.6296875)
\psline[linewidth=0.03cm,arrowsize=0.05291667cm 2.0,arrowlength=1.4,arrowinset=0.4]{<->}(0.5239062,-0.7496875)(1.2839062,-0.7496875)
\psline[linewidth=0.03cm,arrowsize=0.05291667cm 2.0,arrowlength=1.4,arrowinset=0.4]{<->}(1.3239062,-0.1496875)(2.0839062,-0.1496875)
\psline[linewidth=0.03cm,arrowsize=0.05291667cm 2.0,arrowlength=1.4,arrowinset=0.4]{<->}(2.1239064,-0.7496875)(2.8839061,-0.7496875)
\psline[linewidth=0.03cm,arrowsize=0.05291667cm 2.0,arrowlength=1.4,arrowinset=0.4]{<->}(2.903906,-0.1496875)(3.663906,-0.1496875)
\usefont{T1}{ptm}{m}{n}
\rput(0.8023437,-1.0178125){$G_1^2$}
\usefont{T1}{ptm}{m}{n}
\rput(1.6623437,0.1121875){$G_2^2$}
\usefont{T1}{ptm}{m}{n}
\rput(2.4823437,-0.9996875){$G_3^2$}
\usefont{T1}{ptm}{m}{n}
\rput(3.2623436,0.1121875){$G_4^2$}
\psline[linewidth=0.04cm,arrowsize=0.05291667cm 2.0,arrowlength=1.4,arrowinset=0.4]{->}(4.2357817,1.160625)(5.275781,1.160625)
\psline[linewidth=0.04cm,arrowsize=0.05291667cm 2.0,arrowlength=1.4,arrowinset=0.4]{->}(4.2357817,-0.439375)(5.275781,-0.439375)
\usefont{T1}{ptm}{m}{n}
\rput(4.738906,0.650625){Variable}
\usefont{T1}{ptm}{m}{n}
\rput(4.7176566,0.270625){Selection}
\usefont{T1}{ptm}{m}{n}
\rput(6.9114056,0.410625){$\bigcap$}
\usefont{T1}{ptm}{m}{n}
\rput(8.455156,0.390625){=}
\psframe[linewidth=0.0020,dimen=outer,fillstyle=solid,fillcolor=color1120b](10.3957815,-0.819375)(10.083905,-1.239375)
\psframe[linewidth=0.0020,dimen=outer,fillstyle=solid,fillcolor=color1120b](10.3957815,0.320625)(10.083905,-0.059375)
\psframe[linewidth=0.0020,dimen=outer,fillstyle=solid,fillcolor=color1120b](10.3957815,1.900625)(10.083905,1.540625)
\psframe[linewidth=0.04,dimen=outer](10.423905,1.9303125)(10.043906,-1.2696875)
\psline[linewidth=0.04cm](10.063907,1.5303125)(10.403907,1.5303125)
\psline[linewidth=0.04cm](10.063907,1.1303124)(10.403907,1.1303124)
\psline[linewidth=0.04cm](10.063907,0.7303125)(10.403907,0.7303125)
\psline[linewidth=0.04cm](10.063907,0.3303125)(10.403907,0.3303125)
\psline[linewidth=0.04cm](10.063907,-0.0696875)(10.403907,-0.0696875)
\psline[linewidth=0.04cm](10.063907,-0.4496875)(10.403907,-0.4496875)
\psline[linewidth=0.04cm](10.063907,-0.8496875)(10.403907,-0.8496875)
\usefont{T1}{ptm}{m}{n}
\rput(10.182812,1.7003125){1}
\usefont{T1}{ptm}{m}{n}
\rput(10.209685,1.3203125){2}
\usefont{T1}{ptm}{m}{n}
\rput(10.225937,0.9203125){3}
\usefont{T1}{ptm}{m}{n}
\rput(10.219064,0.5203125){4}
\usefont{T1}{ptm}{m}{n}
\rput(10.213438,0.1203125){5}
\usefont{T1}{ptm}{m}{n}
\rput(10.216562,-0.2796875){6}
\usefont{T1}{ptm}{m}{n}
\rput(10.216562,-0.6796875){7}
\usefont{T1}{ptm}{m}{n}
\rput(10.210313,-1.0596875){8}
\usefont{T1}{ptm}{m}{n}
\rput(6.8914056,1.170625){$\{1,2,5,6,7,8\}$}
\usefont{T1}{ptm}{m}{n}
\rput(6.9114056,-0.389375){$\{1,3,5,7,4,8\}$}
\usefont{T1}{ptm}{m}{n}
\rput(8.961406,-1.7034374){$\overline{S}$}
\usefont{T1}{ptm}{m}{n}
\rput(9.58,-1.7434375){$\supset$}
\usefont{T1}{ptm}{m}{n}
\rput(10.171406,-1.7534375){$S^*$}
\end{pspicture} 
}
\caption{An illustration of the Multiple Grouping (MuG) framework.  The true support is $S^* = \{1,5,8\}$ and the estimated superset is $\overline{S} = \{1,5,7,8\}$.}
\label{fig:mugframe}
\end{figure}
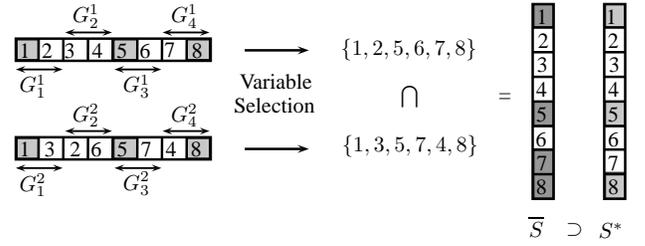

In this Section, we give an overview of the Multiple Grouping (MuG) framework when used in conjunction with an abstract variable selection algorithm.  
Let $V = \{1,\ldots,p\}$ index $\beta^*$ defined in~(\ref{eq:linreg}).  Define a collection of $K$ partitions or groupings of $V$:
\begin{align}
& {\cal G}^{i} = \left\{ G_1^{i},\ldots,G_{d_i}^{i} \right\}\,,\; 1 \le |G_j^{i}| \le m \ll n \label{eq:partone}\\
&\bigcup_{j=1}^{d_i} G_{j}^{i} = V\,,\;
G_{j_1}^{i} \cap G_{j_2}^{i} = \emptyset \,, j_1 \ne j_2 \label{eq:partfour}\,.
\end{align}
We have assumed that each group $G_j^{i}$ has at least one element and at most $m$ elements, where $m$ is small when compared to $n$ and $p$. Moreover, the groups in a grouping ${\cal G}^i$ are chosen such that they are disjoint and all elements in $V$ are mapped to a group in ${\cal G}^i$.  Let $\Alg$ be a generic variable selection algorithm:
\begin{align}
 \widehat{S}^{i} &= \Alg\left(y,X,\lambda,{\cal G}^{i}\right) \label{eq:generalalg}\\
 y &= \text{Observations in (\ref{eq:linreg})} \\
 X &= \text{Design matrix in (\ref{eq:linreg})} \\
 \lambda &= \text{Tuning parameter} \\
 {\cal G}^{i} &= \text{Defined in (\ref{eq:partone})-(\ref{eq:partfour})} \,.
\end{align}
The set $\widehat{S}^{i}$ is an estimate of the true support $S^*$.  We assume that, under certain conditions, $\Alg$ can select all groups $G^i_j$ such that $\beta^*_{G^i_j} \ne 0$.  The multiple grouping (MuG) framework for variable selection is to apply the variable selection algorithm $\Alg$ over multiple groupings ${\cal G}^{i}$ to obtain a sequence of estimates $\{\widehat{S}^{1},\ldots, \widehat{S}^{K}\}$.  The final estimated superset of the support $S^*$ is the intersection of all the estimates.  Algorithm~\ref{alg:muggrouping} summarizes the MuG framework and Figure~\ref{fig:mugframe} illustrates MuG using $K = 2$.

\begin{algorithm}[ht]
\label{alg:muggrouping}
\caption{Multiple Grouping (MuG)}
$\bullet$ Compute $\widehat{S}^{i}$ for $i = 1,\ldots,K$ using (\ref{eq:generalalg}). \\
$\bullet$ Return $\displaystyle{\overline{S} = \bigcap_{i=1}^{K} \widehat{S}^{i}}$.
\end{algorithm}

Typical applications of group based variable selection algorithms assume that it is known a priori which groups of variables in $\beta^*$ are non-zero or zero.  Our setting is different since we assume that $\beta^*$ is sparse (and not necessarily group sparse) and group variables to estimate a superset of the true support.  Since the groupings can be chosen arbitrarily, we repeat this process multiple number of times using different groupings and take an intersection over all the estimates to find the final superset of the true support.  

\smallskip

\noindent
\textbf{Relation to other methods:}
By applying a variable selection algorithm multiple number of times using different groupings, we are introducing randomness into the variable selection algorithm.  Past ways of introducing randomness have relied on subsampling \cite{BachBolasso2008,meinshausen2010stability} or random perturbations of the design matrix $X$ \cite{meinshausen2010stability}.  Our approach of using multiple groupings is a new method for introducing randomness into algorithms for improved performance.

\smallskip

\noindent
\textbf{Choosing $\lambda$:}
The parameter $\lambda$ controls the number of variables selected in each iteration of the MuG framework.  We want to choose $\lambda$ so that all variables in $S^*$ are included in each $\widehat{S}^{i}$ with high probability while $\widehat{S}^i$ is as small as possible.  This will ensure that $S^* \subseteq \overline{S}$ with high probability.  One way of doing this is by carefully choosing $\lambda$ using some model selection algorithm, such as cross-validation, stability selection, of information criterion based methods.  However, this can be computationally challenging.  An alternative approach is to assume an upper bound, say $k'$, for the unknown sparsity level and choose $\lambda$ in each iteration of MuG to select $k'$ groups.  As shown in \cite{liu2009estimation}, the group Lasso can only select at most $n$ groups, so choosing $k' = n$ when using the group Lasso allows for tuning-free screening.  We discuss this algorithm in Section~\ref{sec:MuGGroupLasso}.

\smallskip

\noindent
\textbf{Choosing $K$:}
The parameter $K$ controls the number of groupings we form in the MuG framework.  It is clear that $|\overline{S}|$ decreases or remains the same as $K$ increases.  However, we do not want $K$ to be too large since, with small probability, there may exist a grouping for which we may discard an element in $S^*$.  On the other hand, choosing $K$ to be too small may not result in significant reduction in dimensionality.  We show that when using MuG with group Lasso, choosing $K$ such that  $K/(p/m)^2 \rightarrow 0$ is sufficient to ensure consistency of the screening algorithm.  Thus, choosing $K$ of the order of $\min\{n,p\}$ is sufficient to ensure that screening can be performed with high probability.  In practice, $K$ can be chosen depending on the available computational resources.

\smallskip

\noindent
\textbf{Choosing ${\cal G}^i$:}
We discuss two methods for choosing ${\cal G}^i$.  The first method chooses ${\cal G}^i$ by randomly partitioning the set of indices $V$.  The second method chooses ${\cal G}^{i+1}$ using the estimates $\widehat{S}^{i},\ldots, \widehat{S}^{1} $.  Our numerical simulations compare both these methods and also discusses the trade-offs in choosing $m$, i.e., the maximum size of the groups in ${\cal G}^i$.

\section{MuG Using Group Lasso}
\label{sec:MuGGroupLasso}

So far, we have given an overview of the Multiple Grouping (MuG) framework using an abstract variable selection algorithm.  In this Section, we show how MuG can be used with the group Lasso estimator, which was proposed in \cite{GroupLassoYuanLin2006} as an extension to the Lasso for variable selection and prediction given prior knowledge about groups of variables that are either zero or non-zero.  Section~\ref{subsec:groupLasso} outlines the MuG framework using the group Lasso.  
Section~\ref{subsec:choosinggroups} presents an algorithm for grouping variables and empirically evaluates the algorithm using a simple numerical example.    

\subsection{Tuning-Free Screening}
\label{subsec:groupLasso}
Let ${\cal G}^i$ be a grouping defined in (\ref{eq:partone})-(\ref{eq:partfour}) with $d_i$ groups.  The weighted $(1,\nu)-$norm, using the grouping ${\cal G}^i$, is defined as follows:
\begin{equation}
||\beta||_{{\cal G}^i,\nu} = \sum_{j=1}^{d_i} \sqrt{m_{ij}} \; ||\beta_{G_{j}^{i}}||_{\nu} \,, \label{eq:l1nu}
\end{equation}
where $m_{ij} = |G_j^i|$.  The group Lasso, first proposed in \cite{GroupLassoYuanLin2006}, solves the following optimization problem:
\begin{equation}
\widehat{\beta}^i(\lambda) = \arg \min_{\beta \in \R^p} \left\{ \frac{1}{2n} ||y - X \beta ||_2^2 + \lambda ||\beta||_{{\cal G}^i,2}   \right\} \label{eq:groupLasso} \,.
\end{equation}
When $G^i_j = \{j\}$, the group Lasso reduces to the classical Lasso estimator \cite{TibshiraniLasso1994}.  Let the support over the groups be $\widehat{S}_{{\cal G}^i}(\lambda)$ such that
\begin{equation}
\widehat{S}_{{\cal G}^i}(\lambda) = \{ G_{j}^{i} : \widehat{\beta}_{G_j^i}^i(\lambda) \ne 0 \} \,.
\end{equation}
In other words, $\widehat{S}_{{\cal G}^i}(\lambda)$ is the set of groups in ${\cal G}^i$ selected by the group Lasso estimator.  The following Lemma, proved in \cite{liu2009estimation}, characterizes the cardinality of $\widehat{S}_{{\cal G}^i}(\lambda)$.

\smallskip

\begin{lemma}[\cite{liu2009estimation}]
\label{lemma:groupLassosolution}
For all $\lambda > 0$, $|\widehat{S}_{{\cal G}^i}(\lambda)| \le \min\{n,d_i\}$, where $d_i$ is the number of groups in the grouping ${\cal G}_i$.  
\end{lemma}

\smallskip

\begin{algorithm}[t]
\label{alg:muggroupLasso}
\caption{Tuning-Free Screening}
\begin{itemize}
\item Solve (\ref{eq:groupLasso}) using grouping $G^0_j = \{j\}$ and choose the tuning parameter such that\footnotemark $|\widehat{S}^0(\lambda_0)| = n$. 
\item Initialize $\overline{S} = \widehat{S}^0(\lambda_0)$.
\item For $i = 1,\ldots,K$
\begin{itemize}
\item Choose a grouping ${\cal G}^i$ that satisfies (\ref{eq:partone})-(\ref{eq:partfour}) and $d_i > n$.
\item Solve (\ref{eq:groupLasso}) using ${\cal G}^i$ and choose$^\text{{\scriptsize 2}}$ $\lambda_i$ s.t. $|\widehat{S}_{{\cal G}^i}(\lambda_i)| = n$
\item Let $\widehat{S}^i(\lambda_i)$ be the support of the group Lasso estimator and update $\overline{S}$:
$\overline{S} = \overline{S} \cap  \widehat{S}^i(\lambda_i)$.
\end{itemize}
\end{itemize}
\end{algorithm}
\footnotetext{If $n$ variables can not be selected, the $\lambda_0$ chosen will be $\lambda_0 = \arg \max_{\lambda} |\widehat{S}^0(\lambda)|$.  Similarly, for the group Lasso, if $n$ groups can not be selected, the $\lambda_i$ chosen will be $\lambda_i = \arg \max_{\lambda}|\widehat{S}_{{\cal G}^i}(\lambda)|$.}

Using Lemma~\ref{lemma:groupLassosolution}, we see that the Lasso can select at most $\min\{n,p\}$ variables and the group Lasso can select at most $\min\{n,d_i\}$ groups of variables.  When $p > n$, we can easily perform screening by solving the Lasso to select at most $n$ variables. 

Using MuG, we may further reduce the dimensionality of the problem.  Algorithm~\ref{alg:muggroupLasso} outlines the MuG framework when used in conjunction with the group Lasso estimator in (\ref{eq:groupLasso}).  We first solve the Lasso by choosing a $\lambda$ that selects at most $n$ variables.  If $n$ variables can not be selected, we select the maximum number of variables the Lasso can select.  Next, we solve the group Lasso for multiple different choices of the groupings in such a way that at most $n$ groups are selected.  Again, if $n$ groups can not be selected, we choose the maximum number of groups possible.  The final step is to take an intersection over all the supports to find an estimate $\overline{S}$.  The algorithm is \textit{tuning-free} since we specify \textit{exactly how} the tuning parameters are chosen in the algorithm.  We note that although Algorithm~\ref{alg:muggroupLasso} depends on the parameters $K$ (number of iterations) and $m$ (maximum size of the groups), both these parameters can be easily chosen to allow screening.  We refer to Section~\ref{sec:numsim} for more details.

When using standard implementations of the Lasso and the group Lasso, it may not be computationally feasible for all solutions of the Lasso to have support of size less than or equal to $n$.  Thus, in practice, we apply the Lasso for multiple different values of $\lambda$ and choose a $\lambda$ for which the estimated support is the smallest above $n-1$ .  A similar step is done for the group Lasso solution.  If we apply group Lasso for $C$ different tuning parameters, the worst case computational complexity of Algorithm~\ref{alg:muggroupLasso} is $KC$ times the complexity of computing group Lasso solutions.  In practice, once a suitable tuning parameter has been selected in one iteration, subsequent tuning parameters becomes easier to estimate.  Moreover, different approximations to group Lasso, such as group LARS \cite{GroupLassoYuanLin2006}, may be used to speed up computations.

\subsection{Choosing the Groupings}
\label{subsec:choosinggroups}

\begin{figure}[t]
\begin{center}
\scalebox{0.7} 
{
\begin{pspicture}(0,-1.2507813)(8.66,1.2507813)
\definecolor{color0b}{rgb}{0.8823529411764706,0.8823529411764706,0.8823529411764706}
\psframe[linewidth=0.04,dimen=outer,fillstyle=solid,fillcolor=color0b](3.94,0.7529687)(0.0,0.25296876)
\psframe[linewidth=0.04,dimen=outer](8.66,0.7529687)(3.9,0.25296876)
\usefont{T1}{ptm}{m}{n}
\rput(2.165625,1.2){{\Large $\overline{S}$}}
\usefont{T1}{ptm}{m}{n}
\rput(6.4601564,1.2){\Large $\overline{S}^c$}
\psline[linewidth=0.04cm,arrowsize=0.05291667cm 2.0,arrowlength=1.4,arrowinset=0.4]{->}(2.6,0.43296874)(3.32,-0.6670312)
\psline[linewidth=0.04cm,arrowsize=0.05291667cm 2.0,arrowlength=1.4,arrowinset=0.4]{->}(5.0,0.43296874)(3.78,-0.6670312)
\usefont{T1}{ptm}{m}{n}
\rput(3.5223436,-0.96203125){{\Large ${G_j^i}$}}
\psline[linewidth=0.04cm,arrowsize=0.05291667cm 2.0,arrowlength=1.4,arrowinset=0.4]{->}(5.74,0.45296875)(6.34,-0.68703127)
\psline[linewidth=0.04cm,arrowsize=0.05291667cm 2.0,arrowlength=1.4,arrowinset=0.4]{->}(7.54,0.45296875)(6.76,-0.68703127)
\usefont{T1}{ptm}{m}{n}
\rput(6.528281,-0.96203125){{\Large ${G_{j'}^i}$}}
\end{pspicture} 
}
\end{center}
\caption{Illustration of the adaptive grouping algorithm where we either group one variable from $\overline{S}$ with variables from $\overline{S}^c$ or group variables in $\overline{S}^c$ together.}
\label{fig:adaptivegrouping}
\end{figure}
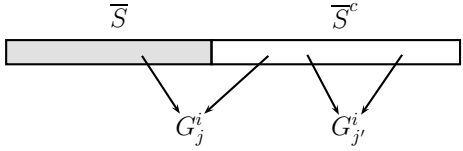

This Section addresses the problem of choosing the groupings ${\cal G}^i$ in the MuG framework.  We consider two methods.

\smallskip

\noindent
\emph{Random Groupings:} Partition the index set randomly such that each group in ${\cal G}^i$ has at most $m$ elements.

\smallskip

\noindent
\emph{Adaptive Groupings:} Let $\overline{S}$ be the current estimate after using MuG $i-1$ times.  Note that all variables in $\overline{S}^c$ are estimated to be zero in $\beta^*$.  To construct the grouping ${\cal G}^i$, randomly group an element in $\overline{S}$ with at most $m-1$ elements from $\overline{S}^c$.  This allows for grouping an element from $\overline{S}$ that is possibly zero with other elements that are estimated to be zero.  Once all elements in $\overline{S}$ have been grouped, randomly group the remaining elements in groups of size at most $m$.  Fig.~\ref{fig:adaptivegrouping} illustrates this adaptive construction.

To compare the performance of the two grouping algorithms, we consider a simple example.  Consider the linear model in (\ref{eq:linreg}) with $w \sim {\cal N}(0,I)$, $p = 100$, $n = 30$, and $k = 5$.  Suppose all non-zero elements in $\beta^*$ have magnitude $1.0$ and each entry in $X$ is sampled independently from a standard normal distribution.  Note that the MuG framework is sensitive to the choice of the $K$ groupings.  A different choice of the groupings may result in a different output $\overline{S}$.  To study the properties of $\overline{S}$, we fix $X$ and $w$ and apply Algorithm~\ref{alg:muggroupLasso} $200$ times over different choices of the $K$ set of groupings for both random and adaptive groupings.  

\begin{figure}[t]
 \begin{center}
\subfigure[Random Grouping]{
 \includegraphics[scale=0.27]{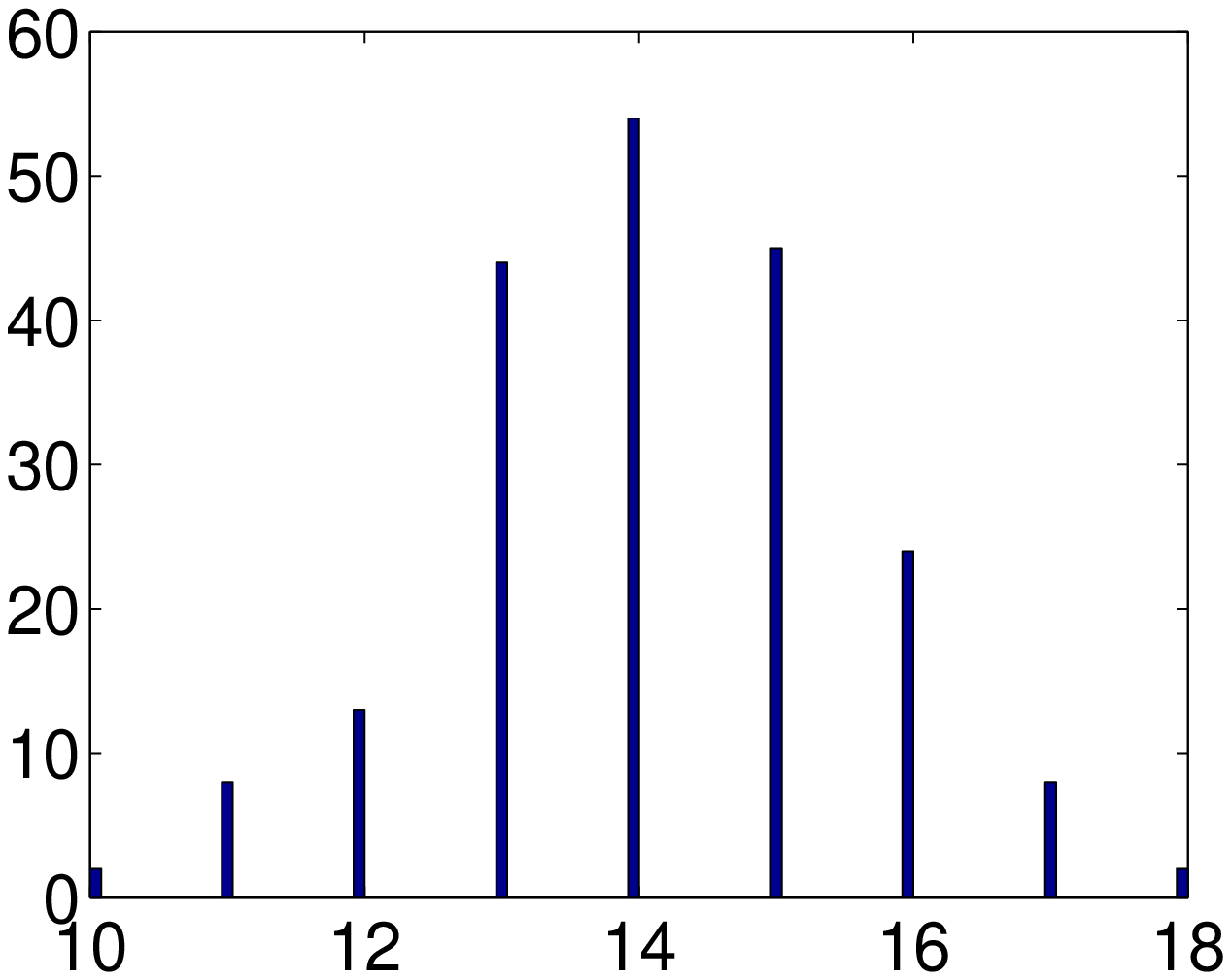}
} 
\subfigure[Adaptive Grouping]{
 \includegraphics[scale=0.27]{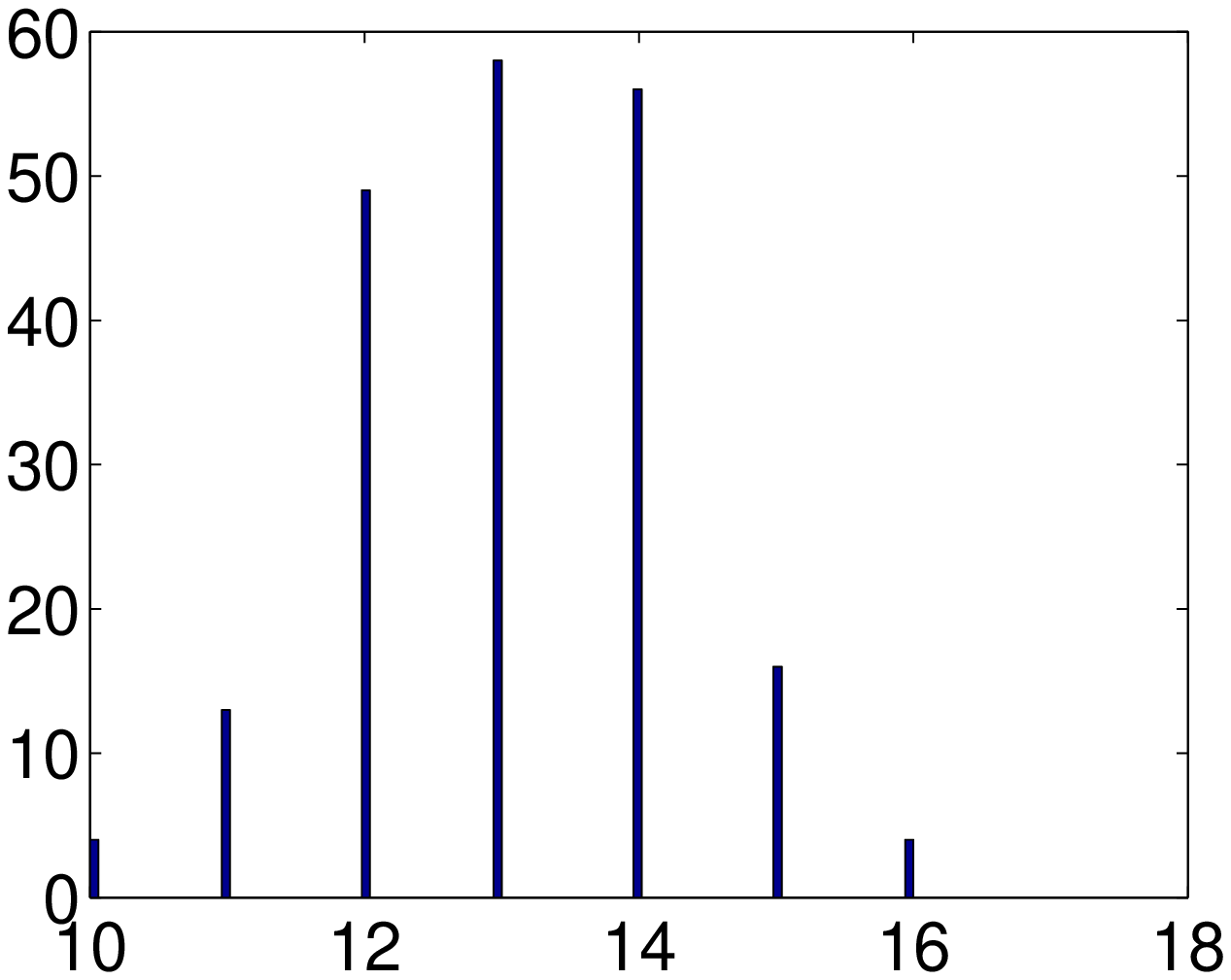}
}
\end{center}
\caption{Histogram of $|\overline{S}|$ when applying MuG $200$ times over different choices of the set of groupings.  The parameters are $p = 100$, $n = 30$, $k = 5$, $\sigma = 1$, $m = 2$, and $K = 50$.  In (a), we choose the groupings by randomly partitioning the index set in groups of size $m$.  In (b), we adaptively choose the groupings in each iteration of the MuG framework as described in Fig.~\ref{fig:adaptivegrouping}.}
\label{fig:mugframeSimpleExample}
\end{figure}

Fig.~\ref{fig:mugframeSimpleExample} shows the histogram of the cardinality of $\overline{S}$.
It is clear that the adaptive groupings approach results in estimates that have lower cardinality than that of the random groupings approach.  For all the $400$ instances ($200$ each for random and adaptive groupings), $\overline{S}$ contained the true support $S^*$.  
We also note that had we not used the group Lasso step in Algorithm~\ref{alg:muggroupLasso} and simply used the Lasso for screening, $n = 30$ variables would have been selected.  On the other hand, using the group Lasso discards on average about half the variables estimated by the Lasso.  This shows the benefits of using the MuG framework as opposed to using Lasso for tuning-free screening. We refer to Section~\ref{sec:numsim} for additional numerical simulations showing the advantages of MuG.

\section{High-Dimensional Consistency}
\label{sec:thprop}
In this Section, we find conditions under which Algorithm~\ref{alg:muggroupLasso} is high-dimensional consistent so that $P(S^* \subseteq \overline{S}) \rightarrow 1$ as $n,p \rightarrow \infty$ and $|\overline{S}| \le n$.  As it turns out, the conditions we require are nearly similar to those required for screening using the Lasso estimator.  Recall the model in (\ref{eq:linreg}) where $X$ and $y$ are known, $\beta^*$ is unknown, $k$ is the number of non-zero entries in $\beta^*$, and $w \sim {\cal N}(0,\sigma^2I)$ is the measurement noise.  Consider the following conditions on $X$ and $\beta^*$.
\medskip
\begin{enumerate}[({A}1)]
\item $\displaystyle{\frac{||X_j||_2^2}{{n}} = 1\,, \quad \text{$\forall$ } j = 1,2,\ldots,p}$.
\smallskip
\item $\displaystyle{\frac{||X \xi||_2^2}{n}  \ge \tau_1 ||\xi||_2^2 - \tau_2 \left(\sqrt{\frac{m}{n}} + \sqrt{\frac{3\log p}{n}} \right)^2  ||\xi||_{{\cal G},2}^2}$ for all groupings ${\cal G}$ satisfying (\ref{eq:partone})-(\ref{eq:partfour}), $\xi \in \R^p$, and constants $\tau_1$ and $\tau_2$.
\smallskip
\item $\displaystyle{\beta_{\text{min}} = \min_{k \in V} |\beta_k^*| > \frac{8 \sigma \sqrt{k}}{\tau} \left(\sqrt{\frac{m}{n}} + \sqrt{\frac{\log p}{n}} \right) }$ for some constant $\tau$ that depends on $\tau_1$ and $\tau_2$.
\smallskip
\item For $\lambda_1 \le \lambda_2$, $\widehat{S}^i(\lambda_2) \subseteq \widehat{S}^i(\lambda_1)$, where $\widehat{S}^i(\lambda)$ is the support estimated by group Lasso.
\end{enumerate}
\smallskip

We assume that the parameter $p$ scales with $n$ so that $p \rightarrow \infty$ as $n \rightarrow \infty$.  Further, $K$ and $k$ are also allowed to scale with $n$.  Our main result for consistency of Algorithm~\ref{alg:muggroupLasso} is stated as follows.

\begin{theorem}[Consistency of Screening]
\label{thm:main}
Under the Assumptions (A1)-(A4), $|\overline{S}| \le n$ and $P(S^* \subseteq \overline{S}) \ge 1 - \frac{c (K+1)}{(p/m)^2}$, where $c$ is a constant.  Further, if $\lim K/(p/m)^2 = 0$, then $P(S^* \subseteq \overline{S}) \rightarrow 1$ as $n,p \rightarrow \infty$, where $\overline{S}$ is computed using Algorithm~\ref{alg:muggroupLasso} and the probability is with respect to the distribution of the noise $w$.
\end{theorem}
\begin{proof}
See Appendix~\ref{subsec:proofmainthm}.
\end{proof}

\smallskip

Theorem~\ref{thm:main} identifies a sufficient condition on the scaling of the parameter $K$.  In particular, we see that choosing $K < p^{\gamma}$, for $\gamma < 2$, is sufficient for consistency of the screening algorithm.  Since we assume $p > n$, in practice we choose an appropriate $K$ so that $K < n$.  We now make some additional remarks regarding Theorem~\ref{thm:main}.

\begin{remark}[Assumption (A1)]
The normalization assumption (A1) is standard and can be easily achieved by scaling the matrix $X$ and correspondingly the vector $\beta^*$.  
\end{remark}

\begin{remark}[Assumption (A2)]
Assumption~(A2), which comes from \cite{negahban2010unified}, ensures restricted strong convexity (RSC) of the least squares loss function in (\ref{eq:groupLasso}) so that the loss function is ``not too flat" around the optimal solution \cite{negahban2010unified}.  We note that alternative conditions based on restricted eigenvalues \cite{bickel2009simultaneous,buhlmann2011statistics}, which are similar to the RSC conditions, may also be assumed instead of (A2).  As shown in \cite{negahban2010unified}, matrices $X$ whos rows are sampled from a multivariate normal distribution satisfy (A2) when given a sufficient number of observations~$n$.
\end{remark}

\begin{remark}[Assumption (A3)]
 Assumption (A3) is a standard condition that imposes a lower bound on $\beta_{\text{min}}$, the minimum absolute value of the non-zero entries in $\beta^*$.  Informally, a small $\beta_{\text{min}}$ requires more observations for consistent estimation using Lasso and group Lasso.  It is interesting to see how $\beta_{\text{min}}$ scales with the group size $m$.  If we do not use MuG and simply use the Lasso for screening, (A3) reduces\footnote{See Lemma~2 in Appendix~A.} to $\beta_{\text{min}} > \frac{8 \sigma}{\tau}  \sqrt{\frac{k\log p}{n}}$.  Using MuG with group Lasso increases the lower bound on $\beta_{\text{min}}$ by $\frac{8 \sigma}{\tau}  \sqrt{\frac{k m}{n}}$.  Thus, although the MuG framework may result in screening such that $|\overline{S}| < n$, this comes at the cost of requiring the minimum absolute value in $\beta^*$ to be slightly larger than that required when simply using the Lasso for screening ($K = 0$ in Algorithm~\ref{alg:muggroupLasso}).  
If $m \ll n$, then the increase in the lower bound on $\beta_{\min}$ is small and converges to $0$ as $n$ grows large.  Thus, it is desirable to choose $m$ as small as possible.  Another motivation for choosing $m$ to be small is so that the convergence rate for the probability of error, which is $O(K m^2/p^2)$, is small as possible.  We will numerically explore different choices of $m$ in Section~\ref{sec:numsim}.
\end{remark}

\begin{remark}[Assumption (A4)] 
Assumption~(A4) ensures that if a variable $v$ is selected for some tuning parameter $\lambda_2$, then $v$ is also selected for any tuning parameter less than $\lambda_2$.  We need (A4) since we always choose the smallest possible tuning parameter when solving the group Lasso problem.  We can easily relax (A4) so that $v \in S^*$.  Furthermore, we can modify Algorithm~\ref{alg:muggroupLasso} in the following way so that (A4) is no longer needed in Theorem~\ref{thm:main} to prove consistency of the screening algorithm:
\begin{itemize}
\item Find the regularization path of the group Lasso by solving (\ref{eq:groupLasso}) for multiple tuning parameters.  Instead of assuming that $\widehat{S}^{i}$ is the support estimated so that group Lasso selects $n$ groups, let $\widehat{S}^{i}$ be the set of all variables selected in the regularization path.
\end{itemize}
\end{remark}

\begin{remark}[Discussion]
From Remarks~1-4, it is clear that (A1)-(A3) are standard assumptions required for proving consistency of the Lasso and the group Lasso estimators.  Assumption~(A1) can be easily achieved by scaling the columns of $X$.  Assumption~(A4) is specific for MuG, but as discussed in Remark~4, it can easily be avoided using a minor modification of Algorithm~\ref{alg:muggroupLasso}.  Thus, (A2) and (A3) are the main assumptions that determine the the success of Algorithm~\ref{alg:muggroupLasso}.  As discussed in Remark~2, there are a wide class of matrices that satisfy (A2) when given an appropriate number of observations.  Assumption~(A3) is satisfied when the non-zero entries in $\beta^*$ have sufficiently large magnitude.  We note that if (A2)-(A3) do not hold, then it is likely that Algorithm~\ref{alg:muggroupLasso} will miss variables from the true support.  As it turns out, since the performance of the Lasso also depends on (A2)-(A3), using the Lasso estimator for screening will also miss variables from the true support.  The same is true for the sure independence screening (SIS) algorithm, which actually requires a stronger condition than (A2) for high-dimensional consistent screening~\cite{fan2008sure}.  In such cases, it is desirable to perform screening in such a way that $\overline{S}$ contains as many variables from the true support as possible.  Using numerical simulations in the next Section on matrices that do not satisfy (A2), we will see that MuG is able to retain more true variables when compared to the Lasso or the SIS algorithm.  Finally, we recall that unlike screening algorithms based on the Lasso or the SIS, Algorithm~\ref{alg:muggroupLasso} has the advantage of not requiring a tuning parameter.
\end{remark}

\section{Numerical Simulations}
\label{sec:numsim}

In this Section, we provide extensive numerical simulations to show the advantages of using MuG in practice.  We assume the linear model in (\ref{eq:linreg}) with $\sigma = 0.5$ and consider three different choices of the $n \times p$ design matrix $X$:

\begin{itemize}

\item (IND) Each entry $X_{ij}$ is sampled independently from ${\cal N}(0,1)$.  We let $p = 1000$ and $n = 100,300,$ or $500$ depending on the example considered.

\item (TOP) Each row in $X$ is sampled independently from ${\cal N}(0,\Sigma)$, where $\Sigma$ is a $p \times p$ covariance matrix that is Toeplitz such that $\Sigma_{ij} = \mu^{|i-j|}$, where we choose $\mu = -0.4$. We let $p = 1000$ and $n = 100, 300,$ or $500$ depending on the example considered.

\item (RL) We use preprocessed data from \cite{li2010inexact}, where $p = 587$ and $n = 148$, such that each row of $X$ corresponds to gene expression values from $p$ genes relating to Lymph node status for understanding breast cancer treatment~\cite{pittman2004integrated}.

\end{itemize}

The matrices in (IND) and (TOP) satisfy the so called mutual incoherence property \cite{ZhaoYuLasso2006,MeinshausenBuhl2006,Wainwright2009} such that exact support recovery is possible using the Lasso given sufficient number of observations.  The matrix in (RL) does not satisfy mutual incoherence, which means that no matter how many observations, the support $S^*$ can not be estimated exactly using Lasso.  We always normalize the columns of $X$ such that $||X_i||_2/\sqrt{n} = 1$.  For each design matrix, we randomly choose $\beta^*$ with a desired sparsity level and a desired $\beta_{\min}$ to simulate the observations $y$.  We emphasize that although we choose the design matrix in (RL) from real data, the actual measurements are being simulated.  This is common practice in the literature for testing the performance of sparse recovery algorithms on real design matrices \cite{meinshausen2010stability}.  We evaluate four possible screening algorithms:

\begin{itemize}
\item MuG: This is our proposed algorithm outlined in Section~\ref{sec:MuGGroupLasso} (see Algorithm~\ref{alg:muggroupLasso}) with the adaptive grouping statregy described in Section~\ref{subsec:choosinggroups}.

\item SIS: This is the sure independence screening algorithm proposed in \cite{fan2008sure}.  Given that the columns are normalized, the algorithm computes $\overline{S}$ by thresholding $\omega = |X^T y|$ such that $\overline{S} = \{i: \omega_i > \tau\}$.  When comparing MuG and SIS, we choose the threshold so that the estimates from both SIS and MuG have the same cardinality.
\item LCV: This is cross-validated Lasso, where we select the Lasso tuning parameter using cross-validation.  We randomly chose $70\%$ of the data for training and the rest for testing and applied Lasso on a grid of values and repeated this process $50$ times.  The final $\lambda$ chosen minimized the mean negative log-likelihood over the training data.  It has been shown theoretically \cite{MeinshausenBuhl2006} and observed empirically that this method may be used to perform screening.  In fact, algorithms such as the adaptive Lasso and the thresholded Lasso use LCV as the first stage in estimating $\beta^*$ and $S^*$. 
\item MuG+LCV: This computes the intersection of MuG and LCV.  The main motivation behind using this method is that since both LCV and MuG result in screening and both the methods are different, the intersection of the results from both these methods can result in a $\overline{S}$ that has lower cardinality.
\end{itemize}

\begin{figure*}
\centering
\subfigure[(IND), $p = 1000, n = 100, k = 10$]{
\includegraphics[scale=0.6]{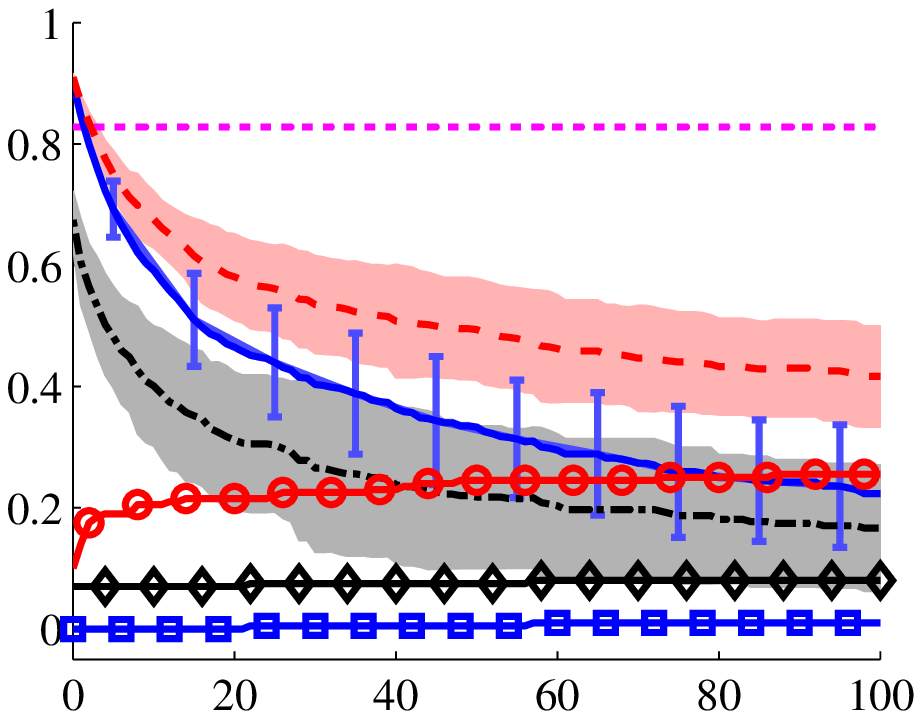}
}
\subfigure[(IND), $p = 1000, n = 300, k = 30$]{
\includegraphics[scale=0.6]{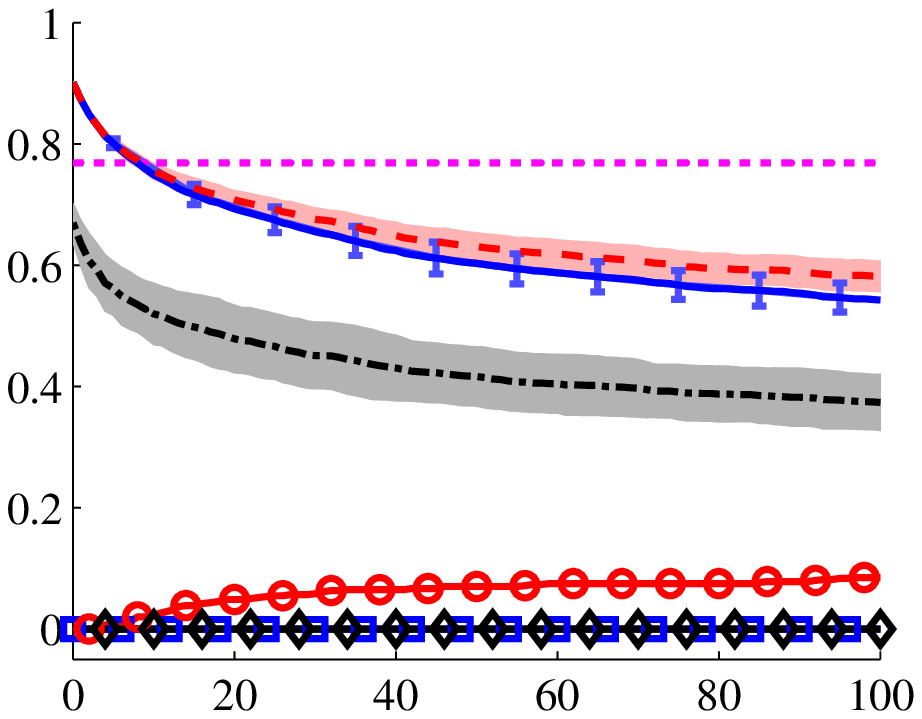}
}
\subfigure[(IND), $p = 1000, n = 500, k = 50$]{
\includegraphics[scale=0.6]{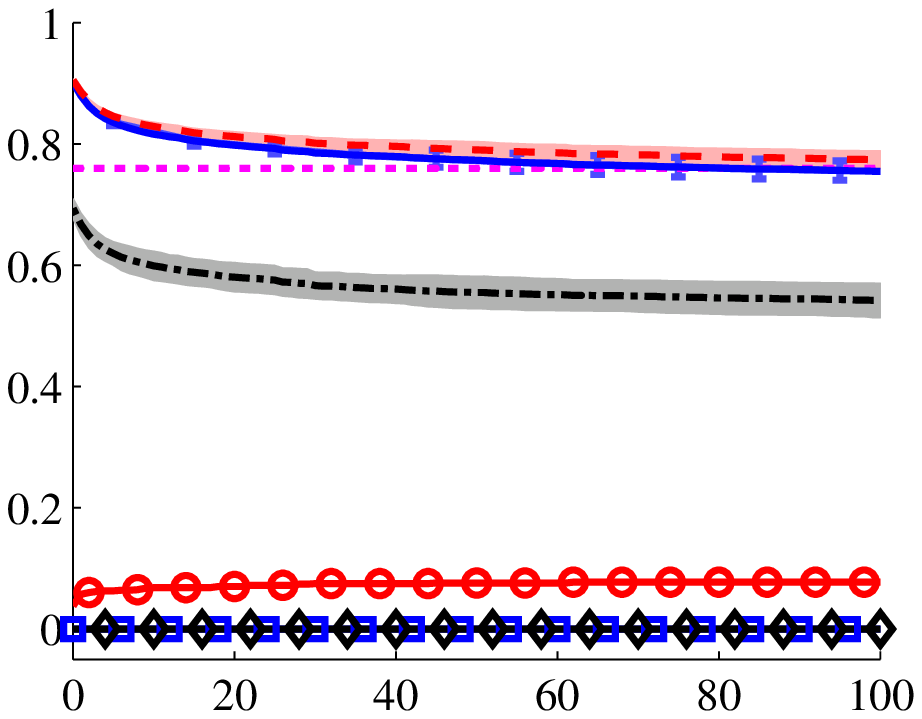}
}
\subfigure[(TOP), $p = 1000, n = 100, k = 10$]{
\includegraphics[scale=0.6]{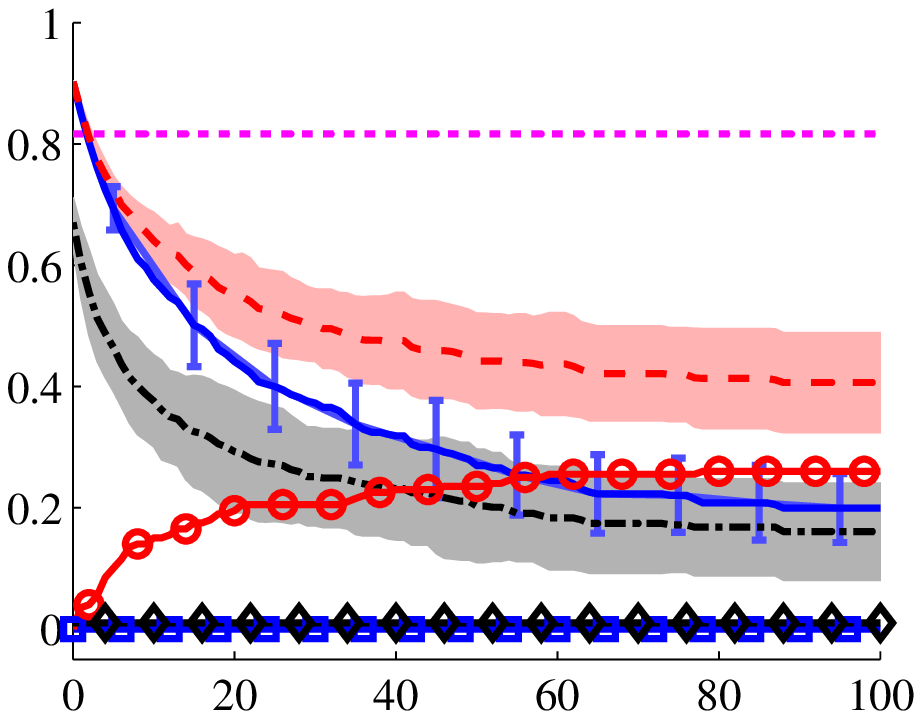}
}
\subfigure[(TOP), $p = 1000, n = 300, k = 30$]{
\includegraphics[scale=0.6]{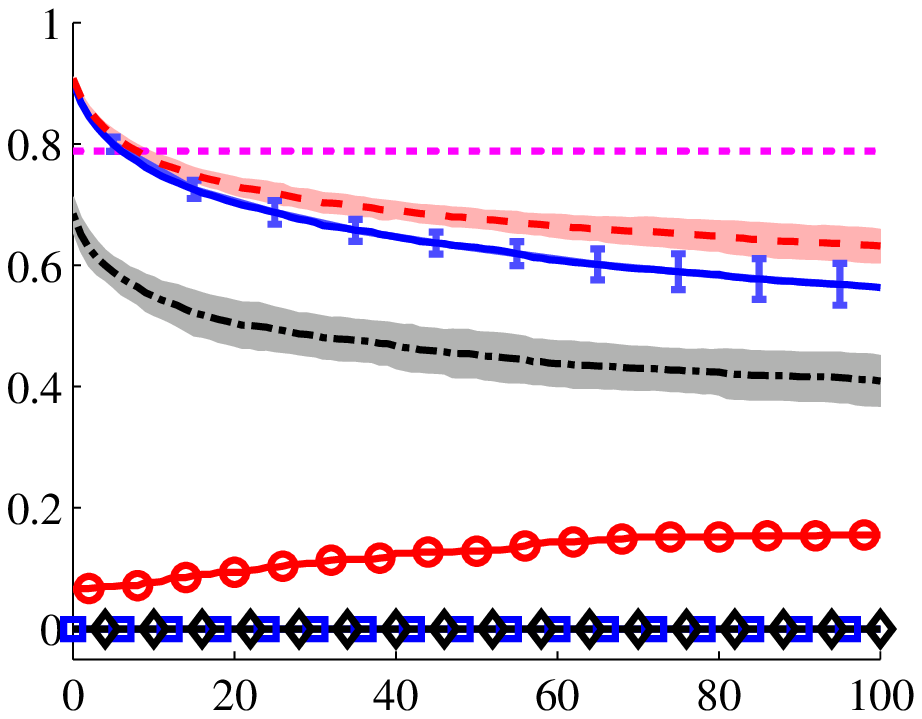}
}
\vspace{0.3cm}
\subfigure[(TOP), $p = 1000, n = 500, k = 50$]{
\includegraphics[scale=0.6]{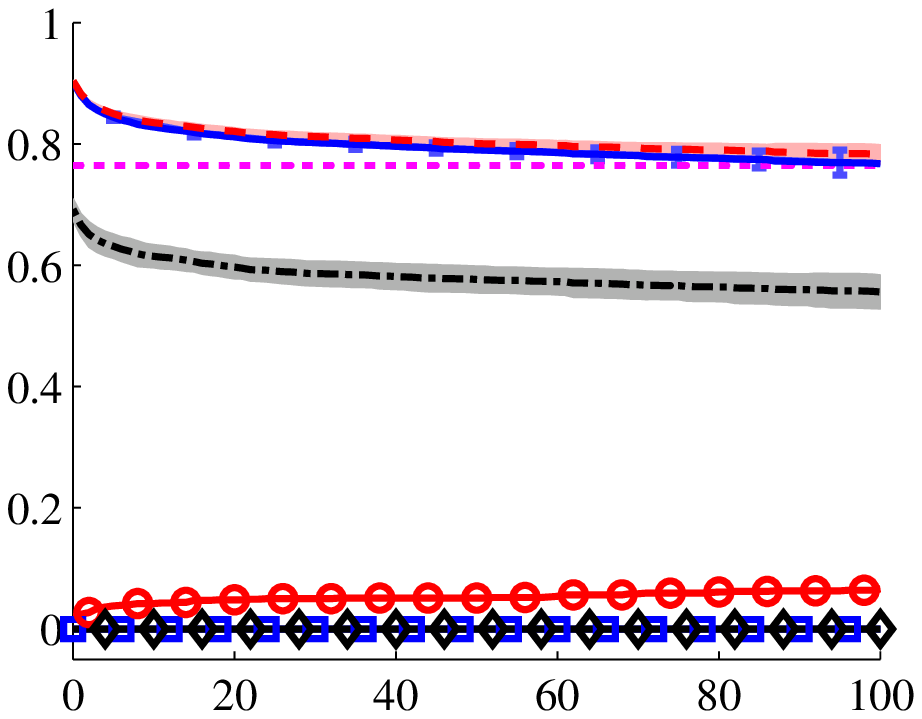}
}
\includegraphics[scale=0.6]{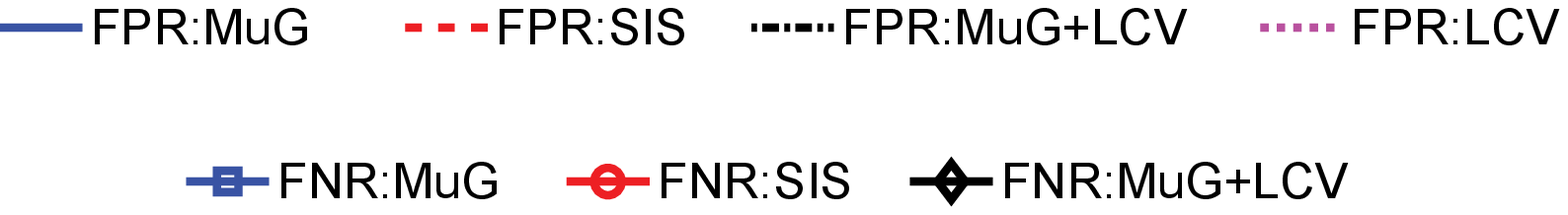}
\caption{Results when $X$ is sampled from a Gaussian distribution.  See Section~\ref{subsec:k} for more details.}
\label{fig:KSim}
\end{figure*}

We evaluate screening algorithms using (i) the fraction of variables in $\overline{S}$ that are not in $S^*$, which we denote by FPR, and (ii)  the fraction of variables in $S^*$ that are not in $\overline{S}$, which we denote by FNR:
\begin{align}
FPR = \frac{|\overline{S} \backslash S^*|}{|\overline{S}|} \;\; \text{and} \;
FNR = \frac{|{S^*} \backslash \overline{S}|}{|S^*|} \,.
\end{align}
In general, we want both the FPR and FNR to be as small as possible.  Section~\ref{subsec:k} discusses results on applying MuG for different choices of the number of groupings $K$.  Section~\ref{subsec:groupsm} discusses results on how the MuG estimates depend on the group size $m$ and the parameter $\beta_{\text{min}}$. 

\subsection{Number of Groupings $K$}
\label{subsec:k}

Figures~\ref{fig:KSim} and \ref{fig:KReal} show results on applying various screening algorithms when $X$ is generated as described by (IND), (TOP), and (RL) and $\beta_{\text{min}} = 0.5$.  The $x$-axis in all figures is the value of $K$ and the $y$-axis is either the FPR or FNR of a screening algorithm.  The lines are mean values of either the FPR or FNR and the shaded regions are the standard deviation of the FPR.  For MuG, we always show the FPR using error bars.  The LCV method is independent of $K$, which is why it does not change with $K$.

\smallskip

\noindent
\textbf{Remarks:}

\begin{enumerate}

\item We clearly see that as $K$ increases, the FPR decreases and the FNR either remains constant or increases at a very small rate.  For cases when $p \gg n$, the FPR decreases at a larger rate and the FNR increases at a small rate.  This is because, when $n \ll p$, MuG removes more variables in each iteration than when $n < p$.

\item We observe that MuG based algorithms perform better than simply using cross-validation (LCV) or using the sure independence screening (SIS) algorithm.  The difference between MuG and SIS is more pronounced in cases where $p$ is much greater than $n$, see for example Fig.~\ref{fig:KSim}(a) and Fig.~\ref{fig:KSim}(d).

\item Combining LCV and MuG, which we refer to as MuG+LCV, leads to a much smaller FPR, while only increasing the FNR by a small amount.  On the other hand, simply using LCV results in a much larger FPR.  For example, in Fig.~\ref{fig:KSim}(a), LCV has a FPR of $0.8$ whereas MuG+LCV has an FPR of about $0.2$.  This shows that the MuG estimates are clearly very different from the Lasso estimates.  We note that in the plots, we do not report the FNR values for LCV since this information can be extracted from the FNR plots for MuG+LCV and MuG.

\item The difference between the performance of MuG and SIS is more pronounced in Fig.~\ref{fig:KReal}, where the matrix $X$ corresponds to real measurements of gene expression values.  For example, in Fig.~\ref{fig:KReal}(a), MuG has an FNR of nearly $0$ and SIS has an FNR of nearly $0.8$.  This means that for the same cardinality of $\overline{S}$, the estimate of MuG contains nearly all the true variables, while SIS is only able to retain $20\%$ of the true variables.  The reason for this significant difference in performance is due to high correlations between the columns of $X$, in which case SIS is known to perform poorly \cite{genovese2012comparison}.  These correlations occur because genes in the same pathway tend to produce correlated measurements \cite{segal2003regression}. 

\end{enumerate}

\begin{figure}
\centering
\subfigure[(RL), $p = 587, n = 148, k = 10$]{
\includegraphics[scale=0.8]{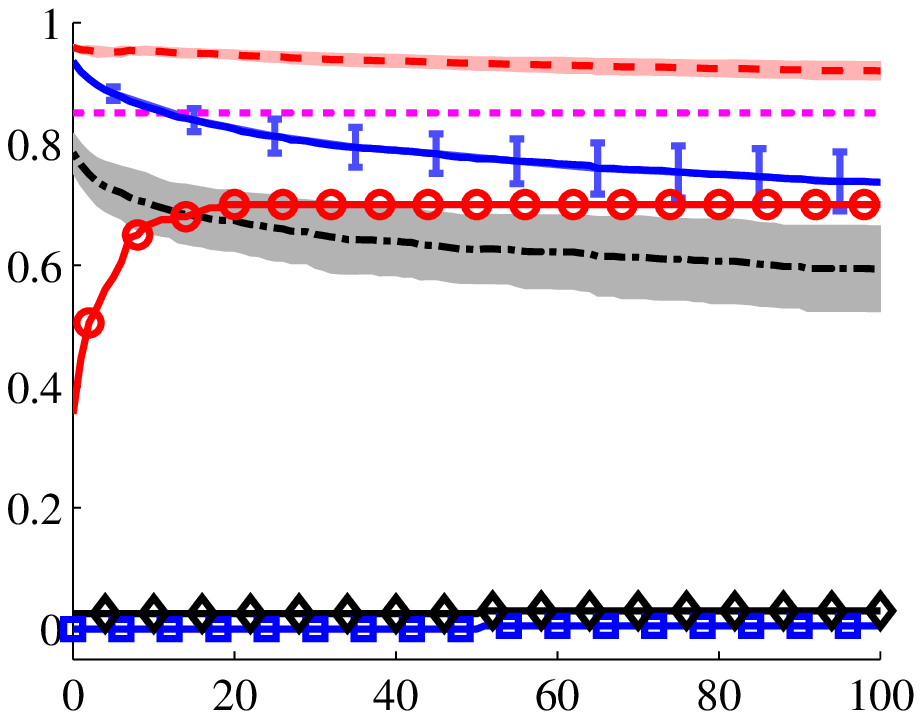}
}
\subfigure[(RL), $p = 587, n = 148, k = 20$]{
\includegraphics[scale=0.8]{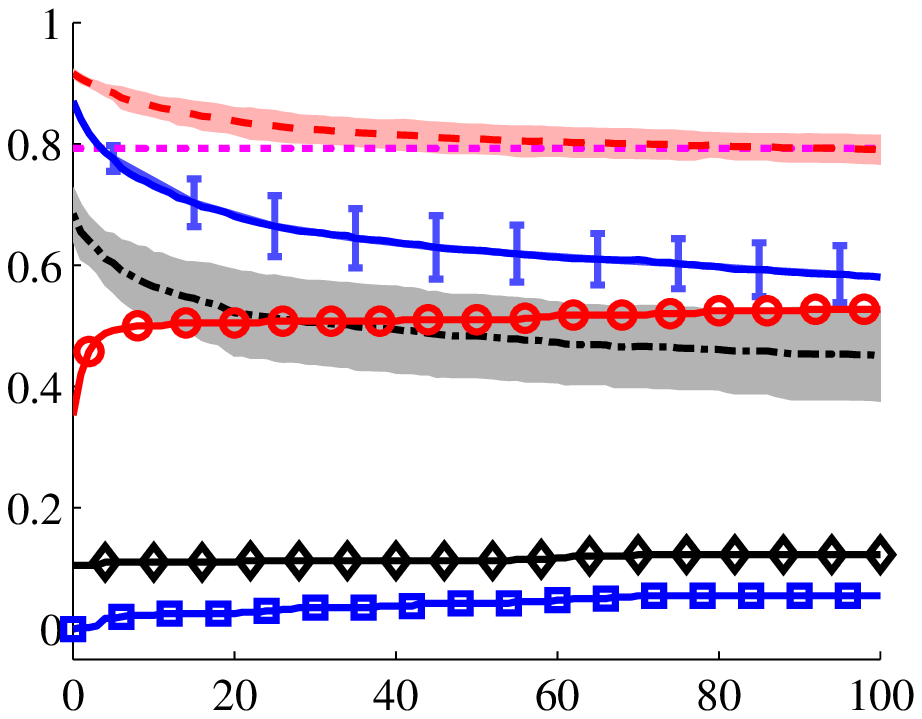}
}
\caption{Results when $X$ is the matrix of gene expression values.  See Figure~\ref{fig:KSim} for the legend and Section~\ref{subsec:k} for more details.}
\label{fig:KReal}
\end{figure}

\subsection{Size of the Groups $m$ and the Parameter $\beta_{\text{min}}$}
\label{subsec:groupsm}

In this Section, we present numerical simulations to study the performance of MuG as the size of the groupings $m$ and the parameter $\beta_{\text{min}}$ change.  Fig.~\ref{fig:exp2and3}(a) shows results for the (IND) example with $p = 1000$, $n = 100$, $k = 10$, and $\beta_{\text{min}} = 2.0$.  We applied MuG using different choices of $m$ ranging from $2$ to $10$ and chose $K = 100$.  As $m$ increases, the mean FPR first decreases and then eventually increases.  The mean FNR increases with $m$.  A similar trend is seen in Fig.~\ref{fig:exp2and3}(b), where $n = 200$.  The only difference is that the number of observations are sufficient for screening, so the FNR is zero as $m$ ranges from $2$ to $6$.  Both these examples show that choosing $m$ to be large does not necessarily result in superior screening algorithms.  In practice, we find that choosing $m=2$ results in good screening algorithms.  

Fig.~\ref{fig:exp2and3}(c) shows results on applying MuG to (IND) and (RL) where we fix all the parameters and vary $\beta_{\text{min}}$.  Only one variable in $\beta^*$ is changed, so it is expected that this particular variable will be difficult to estimate when $\beta_{\text{min}}$ is small.  This is indeed the case from the plot in Fig.~\ref{fig:exp2and3}(c).

\begin{figure*}
\centering
\subfigure[$p = 1000$ and $k = 10$]{
\includegraphics[scale=0.56]{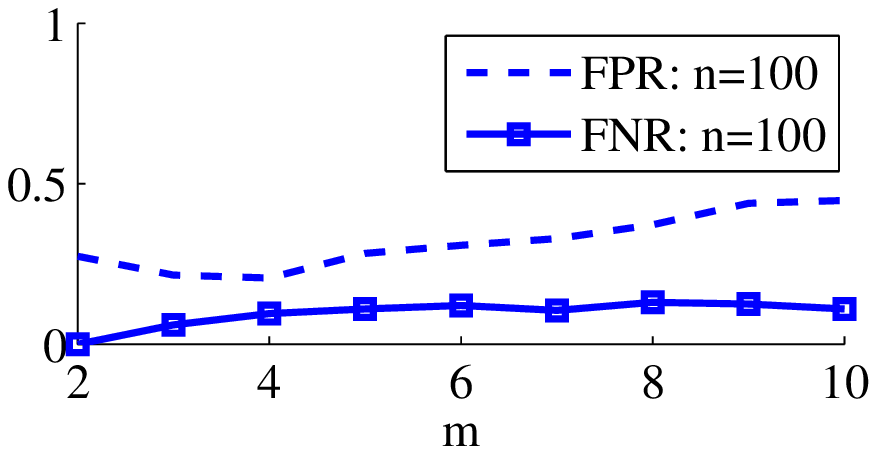}
}
\subfigure[$p = 1000$ and $k = 10$]{
\includegraphics[scale=0.56]{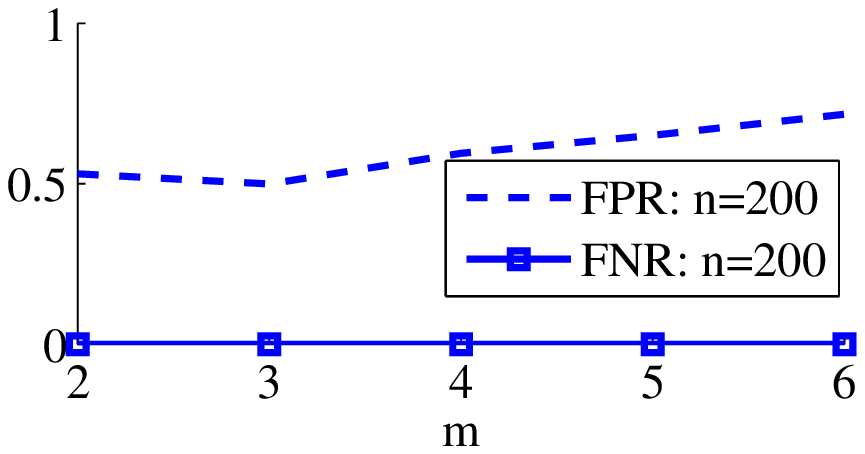}
}
\subfigure[$k = 10$]{
\includegraphics[scale=0.56]{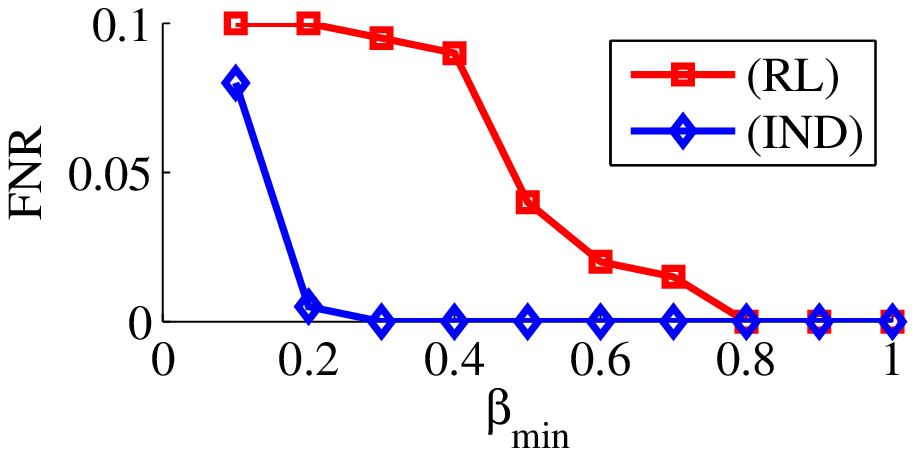}
}
\caption{Performance of MuG as the size of the groupings $m$ and $\beta_{\text{min}}$ change.  See Section~\ref{subsec:groupsm} for more details.}
\label{fig:exp2and3}
\end{figure*}

\section{Extensions and Future Research}
\label{sec:extensions}

We presented the MuG framework in the context of the linear regression problem in (\ref{eq:linreg}) with a sparsity constraint on $\beta^*$.  We now briefly discuss some extensions of MuG along with some future research directions.

\smallskip
\noindent
\textbf{Computational Complexity and Beyond Lasso:} The main focus in this paper was to present the MuG framework for variable screening and analyze it's statistical properties when used with the group Lasso estimator.  Although we saw that using MuG with group Lasso resulted in \textit{tuning free screening}, the potential disadvantage of using the group Lasso is that applying group Lasso multiple number of times may not be computationally feasible for large scale problems.  In such cases, it may be useful to first apply the computationally fast SIS algorithm and then use MuG with group Lasso to further screen variables.  Alternatively, we can also use other group based variable selection algorithms such as group LARS \cite{GroupLassoYuanLin2006}, group marginal regression \cite{bajwa2012group}, cluster representative Lasso (CRL) \cite{buhlmann2013correlated}, block orthogonal matching pursuit \cite{eldar2010block}, or block CoSaMP \cite{baraniuk2010model}.  This will be a subject of future research work.

\smallskip
\noindent
\textbf{Structured Sparsity:}  For many problems, prior knowledge can be useful in constructing better estimates of $\beta^*$.  For example, if it is known that $\beta^*$ is group sparse, group based estimators, such as those in \cite{GroupLassoYuanLin2006,eldar2010block,baraniuk2010model}, can be used to estimate $\beta^*$ using less number of observations. In this case, MuG can be easily applied by forming groupings over the known groups.  For applications in image processing \cite{baraniuk2010model,rao2011convex}, it is natural to assume the variables have a tree-structured sparsity pattern which leads to forming a set of overlapping groups.  Again, the MuG framework can be applied by forming the groupings over the overlapping groups and using algorithms in \cite{jacob2009group} for solving the overlapping group Lasso problem.

\smallskip
\noindent
\textbf{Graphical Model Selection:} A graphical model is a probability distribution defined on graphs.  The nodes in the graph denote random variables and the edges in the graph denote statistical relationships amongst random variables \cite{Lauritzen1996}.  The graphical model selection problem is to estimate the unknown graph given observations drawn from a graphical model.  One possible algorithm for estimating the graph is by solving a Lasso problem at each node in the graph to estimate the neighbors of each node \cite{MeinshausenBuhl2006}.  Our proposed algorithm using MuG in Algorithm~\ref{alg:muggroupLasso} can be used to estimate a superset of the true edges in the graph.  There are many other algorithms in the literature for learning graphical models.  One method, which is commonly referred to as the graphical Lasso \cite{BanerjeeGhaoui2008} or gLasso, solves an $\ell_1$-regularized maximum likelihood problem to estimate a graph.  The MuG framework can be applied to gLasso by placing a group penalty on the inverse covariance.  However, in this case, it is not clear if parameter-free screening can be done.  An alternative method is to assume a conservative upper bound on the number of edges in the graph to perform screening.  Our future work will explore this problem.

\smallskip
\noindent
\textbf{Exact Support Recovery:}  Our primary interest in this paper was screening, i.e., to estimate a superset of the true support.  Exact support recovery can be easily achieved by applying known algorithms for variable selection once screening has been done.  However, it is also of interest to study if exact support recovery or nearly exact support recovery can be achieved using the MuG framework.  This may require assuming some upper bound on the support of $\beta^*$ and then applying MuG with this upper bound.

\section{Summary}
\label{sec:summary}

In this paper, we presented a novel framework for screening, which we refer to as Multiple Grouping (MuG), that groups variables, performs variable selection over the groups, and repeats this process multiple number of times using different choices of the groupings.  The final superset of the true variables is computed by taking an intersection over all the estimated sets over each grouping.  The main advantage of MuG over other screening algorithms is that MuG can perform screening in the linear regression problem without using a tuning parameter.  Theoretically, we proved consistency of the tuning-free screening algorithm and our numerical simulations showed the advantages of using MuG in practice.  We also discussed some future research directions of using MuG in problems involving structured sparsity, graphical models, and exact support recovery.

\section*{Acknowledgement}
We thank Brendan Ames, Jarvis Haupt, Nikhil Rao, Adam Rothman, and Vincent Tan for discussions and comments that improved the quality of the paper.

\appendices

\section{Proof of Theorem~\ref{thm:main}}
\label{subsec:proofmainthm}

The following lemma establishes consistency of the group Lasso solutions.

\begin{lemma}[\cite{negahban2010unified}]
\label{lemma:groupLasso}
Under (A1)-(A3), there exists a $\lambda$ such that the solution $\widehat{\beta}^i(\lambda)$ to the group Lasso problem in (\ref{eq:groupLasso}) with the grouping ${\cal G}^i$ satisfies
\begin{equation}
 ||\widehat{\beta}^i(\lambda) - \beta^*||_2^2 \le \frac{64 \sigma^2 k}{\tau^2}\left(\sqrt{\frac{m}{n}} + \sqrt{\frac{\log p}{n}}\right)^2 
\end{equation}
with probability at least $1 - c_2/(p/m)^2$, where $c_1$ is a constant.
\end{lemma}

Using Lemma~\ref{lemma:groupLasso}, there exists a $\lambda$ such that if $k \le \min\{n,p\}$ and 
\begin{align}
&\beta_{\text{min}}^2 > \frac{64 \sigma^2 k}{\tau^2}\left(\sqrt{\frac{m}{n}} + \sqrt{\frac{\log p}{n}}\right)^2\,,\;\text{then} \label{eq:betamingLasso} \\
&P(S^* \subseteq \widehat{S}^i(\lambda) \, ; \, {{\cal G}^i}) \ge 1 - \frac{c_2}{(p/m)^2} \,.
\end{align}
Choosing $\beta_{\text{min}}$ as in (A3) ensures that (\ref{eq:betamingLasso}) is satisfied.  Under Assumption~(A4), choosing any $\lambda_i < \lambda$ ensures that the support is contained in $S^*$.  Thus, we have
\begin{align*}
P(S^* \subseteq \widehat{S}^i(\lambda_i) \,;\, {{\cal G}^i}) \ge 1 - \frac{c_3}{(p/m)^2} \,, \quad  i = 0,1,\ldots,K \,,
\end{align*}
where $c_3$ is a constant, $\lambda_i$ is the tuning parameter chosen in Algorithm~\ref{alg:muggroupLasso}, and we let ${\cal G}^0 = \{\{1\},\ldots,\{p\}\}$.  To complete the proof, we have
\begin{align}
&P(S^* \subseteq \overline{S} \,;\, \{{\cal G}^0,{\cal G}^1,\ldots,{\cal G}^K\}) \nonumber\\
&= P\left( \bigcap_{i=0}^{K}\{S^* \subseteq \widehat{S}(\lambda_i)\} \,;\, \{{\cal G}^0,{\cal G}^1,\ldots,{\cal G}^K\} \right) \nonumber \\
&= 1 - P\left( \bigcup_{i=1}^{K} \{S^* \not\subseteq \widehat{S}(\lambda_i)\} \,;\, \{{\cal G}^0,{\cal G}^1,\ldots,{\cal G}^K \} \right) \nonumber \\
&\ge 1 - \frac{c_3 (K+1) }{(p/m)^2} \label{eq:jjjj4}
\end{align}
We use the union bound to get (\ref{eq:jjjj4}).  Choosing $K$ such that $\lim \frac{K}{(p/m)^2} \rightarrow 0$ ensures that $P(S^* \subseteq \overline{S} \,;\, \{{\cal G}^0,{\cal G}^1,\ldots,{\cal G}^K \} ) \rightarrow 1$ as $n,p \rightarrow \infty$.
Thus, given a set of groupings $\{{\cal G}^0,{\cal G}^1,\ldots,{\cal G}^K \}$ , we have established consistency of the MuG screening algorithm.  If the groupings are chosen randomly, either using the random grouping or adaptive grouping approaches outlined in Section~\ref{subsec:choosinggroups}, we will still get the same consistency result since the bound in (\ref{eq:jjjj4}) only depends on the maximum size of the group $m$.


\begin{thebibliography}{10}
\providecommand{\url}[1]{#1}
\csname url@samestyle\endcsname
\providecommand{\newblock}{\relax}
\providecommand{\bibinfo}[2]{#2}
\providecommand{\BIBentrySTDinterwordspacing}{\spaceskip=0pt\relax}
\providecommand{\BIBentryALTinterwordstretchfactor}{4}
\providecommand{\BIBentryALTinterwordspacing}{\spaceskip=\fontdimen2\font plus
\BIBentryALTinterwordstretchfactor\fontdimen3\font minus
  \fontdimen4\font\relax}
\providecommand{\BIBforeignlanguage}[2]{{%
\expandafter\ifx\csname l@#1\endcsname\relax
\typeout{** WARNING: IEEEtran.bst: No hyphenation pattern has been}%
\typeout{** loaded for the language `#1'. Using the pattern for}%
\typeout{** the default language instead.}%
\else
\language=\csname l@#1\endcsname
\fi
#2}}
\providecommand{\BIBdecl}{\relax}
\BIBdecl

\bibitem{candes2006robust}
E.~Cand{\`e}s, J.~Romberg, and T.~Tao, ``Robust uncertainty principles: Exact
  signal reconstruction from highly incomplete frequency information,''
  \emph{IEEE Transactions on Information Theory}, vol.~52, no.~2, pp. 489--509,
  2006.

\bibitem{donoho2006compressed}
D.~Donoho, ``Compressed sensing,'' \emph{IEEE Transactions on Information
  Theory}, vol.~52, no.~4, pp. 1289--1306, 2006.

\bibitem{wille2004sparse}
A.~Wille, P.~Zimmermann, E.~Vranov{\'a}, A.~F{\"u}rholz, O.~Laule, S.~Bleuler,
  L.~Hennig, A.~Prelic, P.~Von~Rohr, L.~Thiele \emph{et~al.}, ``Sparse
  graphical gaussian modeling of the isoprenoid gene network in arabidopsis
  thaliana,'' \emph{Genome Biol}, vol.~5, no.~11, p. R92, 2004.

\bibitem{fan2011sparse}
J.~Fan, J.~Lv, and L.~Qi, ``Sparse high dimensional models in economics,''
  \emph{Annual review of economics}, vol.~3, p. 291, 2011.

\bibitem{GroupLassoYuanLin2006}
M.~Yuan and Y.~Lin, ``Model selection and estimation in regression with grouped
  variables,'' \emph{Journal of The Royal Statistical Society Series B},
  vol.~68, no.~1, pp. 49--67, 2006.

\bibitem{TibshiraniLasso1994}
R.~Tibshirani, ``Regression shrinkage and selection via the {L}asso,''
  \emph{Journal of the Royal Statistical Society, Series B}, vol.~58, pp.
  267--288, 1996.

\bibitem{tropp2007signal}
J.~Tropp and A.~Gilbert, ``Signal recovery from random measurements via
  orthogonal matching pursuit,'' \emph{IEEE Transactions on Information
  Theory}, vol.~53, no.~12, pp. 4655--4666, 2007.

\bibitem{ZouAdaptiveLasso2006}
H.~Zou, ``The adaptive {L}asso and its oracle properties,'' \emph{Journal of
  the American Statistical Association}, vol. 101, pp. 1418--1429, December
  2006.

\bibitem{candes2008enhancing}
E.~J. Candes, M.~B. Wakin, and S.~P. Boyd, ``Enhancing sparsity by reweighted
  $\ell_1$-minimization,'' \emph{Journal of Fourier Analysis and Applications},
  vol.~14, no. 5-6, pp. 877--905, 2008.

\bibitem{wasserman2009high}
L.~Wasserman and K.~Roeder, ``High dimensional variable selection,''
  \emph{Annals of statistics}, vol.~37, no.~5A, p. 2178, 2009.

\bibitem{GeerBuhlZhouAdaptive2011}
S.~A. van~de Geer, P.~B{\"u}hlmann, and S.~Zhou, ``The adaptive and the
  thresholded {L}asso for potentially misspecified models (and a lower bound
  for the {L}asso),'' \emph{Electronic Journal of Statistics}, vol.~5, pp.
  688--749, 2011.

\bibitem{zhang2011adaptive}
T.~Zhang, ``Adaptive forward-backward greedy algorithm for learning sparse
  representations,'' \emph{IEEE Transactions on Information Theory}, vol.~57,
  no.~7, pp. 4689--4708, 2011.

\bibitem{belloni2011square}
A.~Belloni, V.~Chernozhukov, and L.~Wang, ``Square-root {L}asso: pivotal
  recovery of sparse signals via conic programming,'' \emph{Biometrika},
  vol.~98, no.~4, pp. 791--806, 2011.

\bibitem{buhlmann2011statistics}
P.~B{\"u}hlmann and S.~van~de Geer, \emph{Statistics for High-Dimensional Data:
  Methods, Theory and Applications}.\hskip 1em plus 0.5em minus 0.4em\relax
  Springer-Verlag New York Inc, 2011.

\bibitem{meinshausen2010stability}
N.~Meinshausen and P.~B{\"u}hlmann, ``Stability selection,'' \emph{Journal of
  the Royal Statistical Society: Series B (Statistical Methodology)}, vol.~72,
  no.~4, pp. 417--473, 2010.

\bibitem{chen2008extended}
J.~Chen and Z.~Chen, ``Extended {B}ayesian information criteria for model
  selection with large model spaces,'' \emph{Biometrika}, vol.~95, no.~3, pp.
  759--771, 2008.

\bibitem{fan2008sure}
J.~Fan and J.~Lv, ``Sure independence screening for ultrahigh dimensional
  feature space,'' \emph{Journal of the Royal Statistical Society: Series B
  (Statistical Methodology)}, vol.~70, no.~5, pp. 849--911, 2008.

\bibitem{fan2009ultrahigh}
J.~Fan, R.~Samworth, and Y.~Wu, ``Ultrahigh dimensional feature selection:
  beyond the linear model,'' \emph{Journal of Machine Learning Research},
  vol.~10, pp. 2013--2038, 2009.

\bibitem{fan2010sure}
J.~Fan and R.~Song, ``Sure independence screening in generalized linear models
  with np-dimensionality,'' \emph{Annals of Statistics}, vol.~38, no.~6, pp.
  3567--3604, 2010.

\bibitem{fan2011nonparametric}
J.~Fan, Y.~Feng, and R.~Song, ``Nonparametric independence screening in sparse
  ultra-high-dimensional additive models,'' \emph{Journal of the American
  Statistical Association}, vol. 106, no. 494, pp. 544--557, 2011.

\bibitem{ke2012covariance}
T.~Ke, J.~Jin, and J.~Fan, ``Covariance assisted screening and estimation,''
  \emph{Arxiv preprint arXiv:1205.4645}, 2012.

\bibitem{tibshirani2011strong}
R.~Tibshirani, J.~Bien, J.~Friedman, T.~Hastie, N.~Simon, J.~Taylor, and
  R.~Tibshirani, ``Strong rules for discarding predictors in {L}asso-type
  problems,'' \emph{Journal of the Royal Statistical Society: Series B
  (Statistical Methodology)}, 2011.

\bibitem{el2011safe}
L.~El~Ghaoui, V.~Viallon, and T.~Rabbani, ``Safe feature elimination for the
  {L}asso,'' \emph{Journal of Machine Learning Research. Submitted}, 2011.

\bibitem{xiang2011learning}
Z.~Xiang, H.~Xu, and P.~Ramadge, ``Learning sparse representations of high
  dimensional data on large scale dictionaries,'' in \emph{Advances in Neural
  Information Processing Systems (NIPS)}, 2011.

\bibitem{xiangfast2012}
Z.~J. Xiang and P.~J. Ramadge, ``Fast {L}asso screening tests based on
  correlations,'' in \emph{2012 IEEE International Conference on Acoustics,
  Speech and Signal Processing (ICASSP)}.\hskip 1em plus 0.5em minus
  0.4em\relax IEEE, 2012, pp. 2137--2140.

\bibitem{osborne2000Lasso}
M.~Osborne, B.~Presnell, and B.~Turlach, ``On the {L}asso and its dual,''
  \emph{Journal of Computational and Graphical statistics}, pp. 319--337, 2000.

\bibitem{liu2009estimation}
H.~Liu and J.~Zhang, ``Estimation consistency of the group {L}asso and its
  applications,'' in \emph{Proceedings of the Twelfth International Conference
  on Artificial Intelligence and Statistics (AISTATS)}, 2009.

\bibitem{BachBolasso2008}
F.~Bach, ``Bolasso: model consistent {L}asso estimation through the
  bootstrap,'' in \emph{Proceedings of the International Conference on Machine
  Learning (ICML)}, 2008.

\bibitem{negahban2010unified}
S.~Negahban, P.~Ravikumar, M.~Wainwright, and B.~Yu, ``A unified framework for
  high-dimensional analysis of $ m $-estimators with decomposable
  regularizers,'' \emph{Statistical Science}, vol.~27, no.~4, pp. 538--557,
  2012.

\bibitem{bickel2009simultaneous}
P.~Bickel, Y.~Ritov, and A.~Tsybakov, ``Simultaneous analysis of {L}asso and
  {D}antzig selector,'' \emph{Annals of Statistics}, vol.~37, no.~4, pp.
  1705--1732, 2009.

\bibitem{li2010inexact}
L.~Li and K.~Toh, ``An inexact interior point method for $\ell_1$-regularized
  sparse covariance selection,'' \emph{Mathematical Programming Computation},
  vol.~2, no.~3, pp. 291--315, 2010.

\bibitem{pittman2004integrated}
J.~Pittman, E.~Huang, H.~Dressman, C.~Horng, S.~Cheng, M.~Tsou, C.~Chen,
  A.~Bild, E.~Iversen, A.~Huang \emph{et~al.}, ``Integrated modeling of
  clinical and gene expression information for personalized prediction of
  disease outcomes,'' \emph{Proceedings of the National Academy of Sciences of
  the United States of America}, vol. 101, no.~22, p. 8431, 2004.

\bibitem{ZhaoYuLasso2006}
P.~Zhao and B.~Yu, ``On model selection consistency of {L}asso,'' \emph{Journal
  of Machine Learning Research}, vol.~7, pp. 2541--2563, 2006.

\bibitem{MeinshausenBuhl2006}
N.~Meinshausen and P.~B{\"u}hlmann, ``{High-dimensional graphs and variable
  selection with the {L}asso},'' \emph{Annals of statistics}, vol.~34, no.~3,
  p. 1436, 2006.

\bibitem{Wainwright2009}
M.~J. Wainwright, ``Sharp thresholds for noisy and high-dimensional recovery of
  sparsity using $\ell_1$-constrained quadratic programming ({L}asso),''
  \emph{IEEE Transactions on Information Theory}, vol.~55, no.~5, May 2009.

\bibitem{genovese2012comparison}
C.~Genovese, J.~Jin, L.~Wasserman, and Z.~Yao, ``A comparison of the {L}asso
  and marginal regression,'' \emph{Journal of Machine Learning Research},
  vol.~13, pp. 2107--2143, 2012.

\bibitem{segal2003regression}
M.~Segal, K.~Dahlquist, and B.~Conklin, ``Regression approaches for microarray
  data analysis,'' \emph{Journal of Computational Biology}, vol.~10, no.~6, pp.
  961--980, 2003.

\bibitem{bajwa2012group}
W.~U. Bajwa and D.~G. Mixon, ``Group model selection using marginal
  correlations: The good, the bad and the ugly,'' in \emph{2012 50th Annual
  Allerton Conference on Communication, Control, and Computing (Allerton)},
  2012, pp. 494--501.

\bibitem{buhlmann2013correlated}
P.~Bühlmann, P.~Rütimann, S.~van~de Geer, and C.-H. Zhang, ``Correlated
  variables in regression: Clustering and sparse estimation,'' \emph{Journal of
  Statistical Planning and Inference}, vol. 143, no.~11, pp. 1835--1858, 2013.

\bibitem{eldar2010block}
Y.~Eldar, P.~Kuppinger, and H.~Bolcskei, ``Block-sparse signals: Uncertainty
  relations and efficient recovery,'' \emph{IEEE Transactions on Signal
  Processing}, vol.~58, no.~6, pp. 3042--3054, 2010.

\bibitem{baraniuk2010model}
R.~G. Baraniuk, V.~Cevher, M.~F. Duarte, and C.~Hegde, ``Model-based
  compressive sensing,'' \emph{IEEE Transactions on Information Theory},
  vol.~56, no.~4, pp. 1982--2001, 2010.

\bibitem{rao2011convex}
N.~S. Rao, R.~D. Nowak, S.~J. Wright, and N.~G. Kingsbury, ``Convex approaches
  to model wavelet sparsity patterns,'' in \emph{IEEE International Conference
  on Image Processing (ICIP)}, 2011, pp. 1917--1920.

\bibitem{jacob2009group}
L.~Jacob, G.~Obozinski, and J.~Vert, ``Group {L}asso with overlap and graph
  {L}asso,'' in \emph{Proceedings of the 26th Annual International Conference
  on Machine Learning (ICML)}, 2009, pp. 433--440.

\bibitem{Lauritzen1996}
S.~L. Lauritzen, \emph{Graphical Models}.\hskip 1em plus 0.5em minus
  0.4em\relax Oxford University Press, USA, 1996.

\bibitem{BanerjeeGhaoui2008}
O.~Banerjee, L.~El~Ghaoui, and A.~d'Aspremont, ``Model selection through sparse
  maximum likelihood estimation for multivariate {G}aussian or binary data,''
  \emph{Journal of Machine Learning Research}, vol.~9, pp. 485--516, 2008.

\end{thebibliography}


\vfill

\end{document}